%% file: main.tex
\documentclass{article}

\PassOptionsToPackage{numbers, compress}{natbib}
\usepackage[preprint]{neurips_2026}

\usepackage[utf8]{inputenc} 
\usepackage[T1]{fontenc}    
\usepackage{hyperref}       
\usepackage{url}            
\usepackage{booktabs}       
\usepackage{multirow}
\usepackage{amsfonts}       
\usepackage{nicefrac}       
\usepackage{microtype}      
\usepackage{xcolor}         

\usepackage{wrapfig}
\usepackage{ifthen}
\usepackage{microtype}
\usepackage{graphicx}
\usepackage{subcaption}
\usepackage{bm}
\usepackage[normalem]{ulem}
\usepackage[dvipsnames]{xcolor}
\usepackage[ruled]{algorithm2e}
\usepackage{amsmath}
\usepackage{amssymb}
\usepackage{mathtools}
\usepackage{amsthm}
\usepackage[textsize=tiny]{todonotes}

\usepackage{booktabs} 

\usepackage{hyperref}

\usepackage[capitalize,noabbrev]{cleveref}
\input{custom_commands}

\input{macros}

\title{
To discretize continually:\\ Mean shift interacting particle systems\\ for Bayesian inference
}

\author{%
  Ayoub Belhadji\thanks{Equal contributions. Author emails: abelhadj@mit.edu, dannys4@mit.edu, ymarz@mit.edu}\\
  \And
  Daniel Sharp\footnotemark[1]\\[6pt]
  Center for Computational Science and Engineering \\
  Laboratory for Information and Decision Systems \\
  Massachusetts Institute of Technology\\
  \And
  Youssef Marzouk \\
}

\begin{document}

\maketitle

\begin{abstract}
\input{sections/00_abstract}

\end{abstract}

\input{sections/01_introduction}

\input{sections/02_background}

\input{sections/03_results}

\input{sections/04_experiments}

\input{sections/05_conclusions}

\newpage
\bibliography{bibs/abelhadj,bibs/dannys4}

\appendix
\onecolumn
\input{sections/appendix}

\clearpage

\end{document}

%% file: custom_commands.tex
\newcommand{\EE}{\mathbb{E}}

\newcommand{\gX}{\mathcal{X}}

\newcommand{\piquad}{{p}}

\newcommand{\Y}{Y}
\newcommand{\Yast}{Y^\ast}
\newcommand{\yi}{y_{i}}
\newcommand{\yj}{y_{j}}
\newcommand{\ym}{y_{m}}
\newcommand{\Yt}{\Y^{(t)}}
\newcommand{\Ytp}{\Y^{(t+1)}}

\newcommand{\w}{\bm{w}}
\newcommand{\wY}{\bm{w}^{\ast}(\Y)}
\newcommand{\wlY}{\bm{w}_{\lambda}^{\ast}(\Y)}

\newcommand{\Wast}{\bm{W}^\ast}
\newcommand{\wast}{\bm{w}^\ast}
\newcommand{\WY}{\bm{W}^{\ast}(\Y)}

\newcommand{\WLY}{\bm{W}^{\ast}_{\lambda}(\Y)}

\newcommand{\piquadYw}{\piquad(\w,\Y)}

\newcommand{\KY}{\bm{K}(\Y)}

\newcommand{\vzero}{v_0}
\newcommand{\vzerob}{\bm{v}_0}
\newcommand{\vone}{v_1}

\newcommand{\voneb}{\bm{v}_1}

\newcommand{\dpix}{\,\mathrm{d}\pi(x)}

\newcommand{\MMD}{\mathrm{MMD}}
\newcommand{\MSIP}{\mathrm{MSIP}}

\newcommand{\tocite}[1][]{[\textbf{\ifthenelse{\equal{#1}{}}{???] }{#1]}}}

\theoremstyle{plain}
\newtheorem{theorem}{Theorem}[section]
\newtheorem{proposition}[theorem]{Proposition}

\theoremstyle{definition}

\theoremstyle{remark}

%% file: macros.tex
\newcommand{\RR}{\mathbb{R}}

\newcommand{\X}{\mathcal{X}}

\newcommand{\MMS}{\mathrm{MSIP}}

\newcommand{\wi}{w_{i}}

\newcommand{\yij}{y_{ij}}
\newcommand{\ymsj}{y_{mj}}

\newcommand{\yone}{y_{1}}

\newcommand{\yms}{y_{m}}

\newcommand{\xj}{x_{j}}

\newcommand{\vzerov}{\bm{v}_{0}}

\newcommand{\hvzero}{\hat{v}_{0}}

\newcommand{\vonev}{v_{1}}
\newcommand{\vonem}{\bm{v}_{1}}

\newcommand{\hvonev}{\hat{v}_{1}}

\newcommand{\Kyl}{\bm{K}_{\lambda}(\Y)}

\newcommand{\bmW}{\bm{W}(\Y)}

\newcommand{\what}{\hat{\bm{w}}(\Y)}

\newcommand{\whati}{\hat{w}_{i}(\Y)}
\newcommand{\whatms}{\hat{w}_{m}(\Y)}

\newcommand{\Dij}{\nabla_{ij}}
\newcommand{\bij}{\bm{b}_{ij}}
\newcommand{\Ki}{\bm{K}_i}
\newcommand{\Km}{\bm{K}}

\theoremstyle{plain}
\theoremstyle{definition}
\theoremstyle{remark}

\crefname{corollary}{Corollary}{Corollaries}

%% file: sections/00_abstract.tex
Integration against a probability distribution given its unnormalized density is a central task in Bayesian inference and other fields. We introduce new methods for approximating such expectations with a small set of
weighted samples---i.e., a quadrature rule---constructed via an interacting particle system that minimizes maximum mean discrepancy (MMD) to the target distribution. These methods extend the classical mean shift algorithm, as well as recent algorithms for optimal quantization of empirical distributions, to the case of continuous distributions. Crucially, our approach creates dynamics for MMD minimization that are invariant to the unknown normalizing constant; they also admit both gradient-free and gradient-informed implementations. The resulting mean shift interacting particle systems converge quickly, capture anisotropy and multi-modality, avoid mode collapse, and scale to high dimensions. We demonstrate their performance on a wide range of benchmark sampling problems, including multi-modal mixtures, Bayesian hierarchical models, PDE-constrained inverse problems, and beyond.

%% file: sections/01_introduction.tex
\section{Introduction}\label{sec:intro}
The need to compute expectations with respect to a target probability distribution $\pi$ is ubiquitous in statistical inference and machine learning.
Most distributions, however, do not admit any closed-form numerical integration technique. This situation often arises in Bayesian inference, where posterior 
quantities of interest are expressed as integrals over a distribution whose density is only known up to normalization; similar situations arise in computational chemistry and other physical sciences.

Recent work has approached integral discretization through gradient flows on the space of probability measures~\cite{San17,ChNiRi24}.
In this framework, a discretization is identified with a probability measure, 
often an empirical measure supported on finitely many particles, 
and is characterized as the solution of a variational problem. One then constructs a gradient flow to minimize the chosen discrepancy between this measure and the target distribution $\pi$. Such formulations naturally give rise to partial differential equations (PDEs) which may be discretized via systems of interacting particles~\cite{LiWa16,GiHoLiSt20,GaNuRe20,CaHoSt22,MaMa24,WaNu24}.

Crucially, the most useful gradient flow methods have dynamics invariant to the unknown normalizing constant of the target density; this feature is usually due to the choice of discrepancy, e.g., Kullback--Leibler (KL) divergence, $\chi^2$ divergence, or kernelized Stein discrepancy (KSD).
Designing particle systems as discretizations of gradient flows leads to algorithms with strong theoretical guarantees largely concerned with the mean-field limit, i.e., when the number of particles goes to infinity. The practical, non-asymptotic regime (with finite particles) still presents numerous challenges, especially for target distributions exhibiting strong non-Gaussianity, local anisotropy, and multi-modality.

In this paper, we study integration through the lens of \emph{quadrature rules}. A quadrature rule approximates the expectation of a function $f:\mathcal{X}\to\mathbb{R}$ by a weighted average of function evaluations,
\begin{equation}
\mathbb{E}_\pi[f(Y)] = \int_{\mathcal{X}}f(y) \, \mathrm{d}\pi(y) \approx \sum\limits_{i=1}^{M} w_{i} f(y_i),
\end{equation}
where the weights $w_i \in \mathbb{R}$ and the nodes $\yi \in \mathcal{X}$. From this perspective, most interacting particle systems use \emph{uniform} weights, a design choice that constrains their effectiveness as numerical integration schemes. We instead develop \emph{weighted} quadrature rules for essentially arbitrary unnormalized densities. This task is nontrivial, as classical quadrature constructions rely on estimating high-order moments~\cite{GoWe69,Gau81}, which is generally intractable for unnormalized densities. 

We show that effective quadrature rules for unnormalized densities can be designed by minimizing a regularized maximum mean discrepancy (MMD) associated with the squared exponential kernel. 
This result is unexpected: MMD minimization is typically viewed as intractable for high-dimensional unnormalized densities.
Our approach circumvents this obstacle by introducing a damped fixed-point iteration whose dynamics are invariant, by construction, to the unknown normalizing constant.
Our {main contributions} are as follows:
\begin{itemize}

    \item We extend an optimal quantization method from the setting of empirical target distributions to continuous target distributions known through their unnormalized densities, as in sampling and Bayesian inference.

    \item We introduce finite-particle dynamics for MMD minimization in this setting, and prove that they are invariant to the normalization of the target despite the fact that MMD itself depends on this constant.

    \item We show that certain regularization techniques needed for stability of our kernel implementation are cleanly interpretable as optimization of a \textit{regularized} MMD.%

    \item We introduce multiple \textit{gradient-free} and \textit{gradient (score)-informed} formulations of our algorithm, suitable for a broad range of application domains.

    \item We demonstrate that our approach is scalable to high dimensions and to diverse, geometrically challenging (e.g., strongly multi-modal and anisotropic) target distributions.
    
\end{itemize}

\begin{figure}[t!]
    \includegraphics[width=0.95\linewidth,clip,trim={0 0 0 0}]{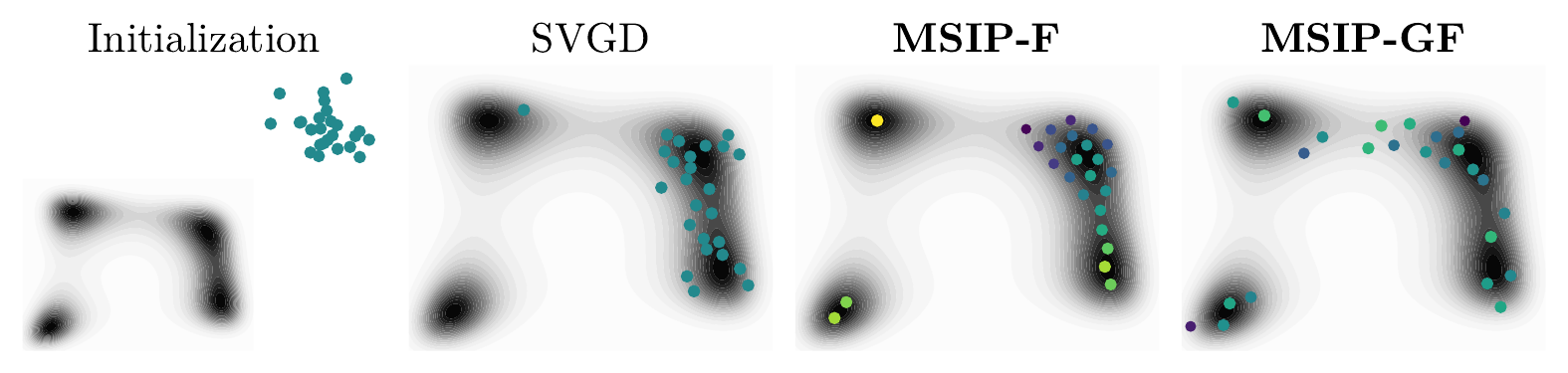}
    \caption{Example of our methods, with color indicating particle weight. (Left): Particle initialization far from support. (Middle left): Result of SVGD with 100 target gradient evaluations per particle. (Middle right): Our gradient-informed method with 100 target gradient evaluations per particle. (Right): Our gradient-free method with 500 target density evaluations per particle.}
    \label{fig:himmelblau_title}
\end{figure}

\paragraph{Notation}
For a matrix $\bm{A}\in\RR^{m\times n}$ with elements $[\bm{A}]_{ij} = a_{ij}$, we write the matrix slice as $[\bm{A}]_{i,:} \coloneqq (a_{i1},a_{i2},\ldots,a_{in})$. Throughout the paper, we consider a domain of interest $\gX\subseteq\RR^d$ and a density $\pi:\gX\to\RR^+$. Then, we use a set of particles $\{ \yone,\ldots,\ym \}\subset\gX$ in a configuration denoted $\Y\in\gX^M$ with $[\Y]_{i,:} := \yi$. 
For $\mathbf{x}\in\RR^M$, we use the notation $\mathrm{diag}(\mathbf{x})\in\RR^{M\times M}$ to denote the diagonal matrix whose $(i,i)$-th entry is $x_i$.
Given a symmetric positive-definite kernel $\kappa:\gX\times\gX\to\RR^+$, we define the kernel matrix $\bm{K}:\gX^M\to\RR^{M\times M}$ via $[\KY]_{ij} = \kappa(\yi,\yj)$. We write the \textit{kernel mean embedding} $\vzerob:\gX^M\to\RR^{M}$ of a density $\pi$ as
\begin{equation}\label{eq:vzero_def}
\vzero(y) = \int\kappa(x,y)\,\dpix,\quad [\vzerob(\Y)]_i = \vzero(\yi).
\end{equation}
Similarly, we define the \textit{kernel embedding of the first moment} $\voneb:\gX^M\to\RR^{M\times d}$ as
\begin{equation}\label{eq:vone_def}
\vone(y) = \int x\kappa(x,y)\,\dpix,\quad [\voneb(\Y)]_{i,:} = \vone(\yi).
\end{equation}

%% file: sections/02_background.tex
\section{Theory and background}
We contextualize our approach by reviewing literature on integration given an unnormalized density. 

\subsection{From MCMC methods and Langevin dynamics to interacting particle systems}
The overdamped Langevin diffusion admits an invariant law with density $\pi$. Discretizing this diffusion leads to Markov chain Monte Carlo (MCMC) algorithms such as ULA and MALA~\cite{Bes94,RoTw96}, which incorporate gradients of the log-density and enable more efficient exploration. In the latter case, Metropolization~\cite{Nea96,RoCaCa99} ensures that the target distribution is preserved after discretization. There are of course myriad other sampling algorithms using Markov chains for Bayesian inference and generative modeling, including both gradient-free~\cite{GoWe10,HaLaMiSa06} and gradient-based~\cite{DuKePeRo87,HoGe14} approaches. Yet these algorithms often struggle on the geometrically complex target distributions described in \S\ref{sec:intro}.

\citet{JoKiOt98} interpret the Langevin dynamics as the gradient flow of the KL divergence in the Wasserstein geometry. The gradient flow is expressed as a PDE governing the evolution of a probability density $\rho_{t}$ for which $\mathrm{KL}(\rho_{t}\|\pi)$ decreases with $t$. 
This variational point of view has since been extended to alternative geometries and divergences.
These variational approaches emphasize 
using an \emph{ensemble of interacting particles} instead of independent Markov chains.

\paragraph{Stein variational gradient descent}
SVGD is derived as a gradient flow of the KL divergence under the Stein geometry~\cite{DuNuSz23}, which is induced by a reproducing kernel Hilbert space (RKHS)~\cite{LiWa16}. At each iteration, a nonlinear map defined by the kernel of the RKHS is applied to the set of particles. %
This algorithm has since been the subject of extensive theoretical investigation. 
Convergence of mean-field, i.e., infinite-particle, SVGD dynamics in various divergences and discrepancies has been discussed thoroughly~\cite{Liu17,KoSaArLuGr20,ChLeLuMaRi20,SaSuRi22,SuKaRi23,DuNuSz23,HeBaSrLu25}. More recently, \cite{ShMa23,LiGhBaPi23, DaNa23, BaBaGh25} obtain convergence guarantees for the finite-particle SVGD algorithm.
Despite this theory, standard SVGD empirically suffers from mode collapse, often failing to cover all modes of multi-modal target distributions~\cite{ZhLiShZhChZh18}.

\paragraph{Interacting particle Langevin samplers}
Inspired by methods for Bayesian filtering~\cite{Eve03}, ensemble Kalman approaches have been introduced for gradient-free sampling in the Bayesian setting, particularly with Gaussian observation models. In the mean-field limit, certain ensemble Kalman samplers (EKS) realize gradient flows of the KL divergence with respect to the so-called Kalman--Wasserstein geometry. The resulting PDE is a Fokker--Planck equation preconditioned by the covariance matrix of the evolving density~\cite{GiHoLiSt20}. EKS also admits affine-invariant variants, e.g., ALDI~\cite{GaNuRe20}, a crucial property for highly anisotropic target densities where standard Langevin dynamics converge slowly. 

\paragraph{Kernel Fisher--Rao flow}
Instead of using the Wasserstein geometry as in~\citep{JoKiOt98}, a few recent  algorithms propose using gradient flow in the \textit{Fisher--Rao} geometry~\citep{WaNu24,MaMa24}. These methods follow a log-geometric path $\log\pi_\tau := (1-\tau)\log\pi_0 + \tau\log\pi_1$, $\tau\in[0,1]$, which connects an initial $\pi_0$ and a target $\pi_1$. While such bridges between prior and posterior have been used in Bayesian inference for many years~\cite{GeTh95}, schemes in \citep{WaNu24,MaMa24} implement either gradient-free and gradient-informed estimators of particle transport using kernel methods, and allow practitioners to choose a time integration scheme over $\tau$. The theory of Fisher--Rao gradient flows has also seen much recent development \cite{ChHuHuReSt23,ZhMi24,CaChHuHuWe26}.

\paragraph{Consensus-based methods}

Consensus-based sampling (CBS) comprises a system of interacting particles designed to sample from an unnormalized density~\cite{CaHoSt22}; it builds on consensus-based optimization methods for gradient-free non-convex optimization~\cite{PiToTsMa17} and relates to collective behavior modeling~\cite{KeEb95,DoBl05,Tos06,CuSm07,CaFoToVe10}. At each iteration, particles are driven toward a consensus point, a weighted average favoring lower objective values, while exploring the space through additive noise. CBS is amenable to theoretical study from the mean-field perspective~\citep{CaChToTs18,HuQiRi23,GeHoVa25,KoWeZe25}, with explicit convergence rates, though these guarantees are generally limited to unimodal densities~\cite{CaHoSt22,MaMa24}.

\subsection{MMD--optimizing methods}

MMD is widely used in numerical analysis to assess quadrature rules~\cite{Hic98,BaLaOb12,MuFuSrSc17,MoPhDiBh18}.
Using the reproducing kernel $\kappa$ for the RKHS $\mathcal{H}$, the MMD is given by
\begin{equation}\label{eq:mmd_def}
    \MMD^2(\mu,\nu) \coloneqq\mathbb{E}_{\mu,\mu}[\kappa(X,X^\prime)]-2\EE_{\mu,\nu}[\kappa(X,Y)]+\EE_{\nu,\nu}[\kappa(Y,Y^\prime)] \, .
\end{equation}
Unlike the KL divergence, MMD is well defined even for mutually singular measures, as when comparing a continuous distribution to a discrete empirical distribution. %
Indeed, for a quadrature rule expressed as a discrete measure $\piquadYw\coloneqq \sum_{i=1}^M w_i\delta_{\yi}$, the worst-case integration error over the unit ball of the RKHS, given by
\begin{equation*}
 \sup_{\|f\|_{\mathcal{H}}\leq 1}\bigg|\int_{\mathcal{X}} f(x) \mathrm{d} \pi(x) - \sum\limits_{i=1}^{M} w_{i} f(y_i)\bigg|,
\end{equation*}
coincides with $\mathrm{MMD}( \pi, \piquadYw)$.
This correspondence has motivated minimizing $\MMD(\pi,\piquadYw)$ using different approaches, including ridge leverage score sampling \cite{Bac17}, determinantal point processes \cite{BeBaCh19,Bel21}, continuous volume sampling \cite{BeBaCh20}, Fekete points \cite{KaSaTa21}, recombination algorithms \cite{HaObLy22,HaObLy23},
and gradient flows \cite{ArKoSaGr19,GlDvMiZh24}. These methods rely on symmetric kernels and, unfortunately,
assume a normalized density. %
Using~\eqref{eq:mmd_def} for density $\tilde{\pi} = Z_{\tilde{\pi}}\pi$, we see
\begin{equation}
    \MMD^2(\tilde{\pi},\piquadYw) = C_{\tilde{\pi}} - 2Z_{\tilde{\pi}}\EE_{\pi}\left[\sum_{i=1}^M \wi\kappa(X,\yi)\right] + \sum_{i,j=1}^{M}\wi w_j\kappa(\yi,\yj),
\end{equation}
for some constant $C_{\tilde{\pi}}$ independent of $p$. This affine relationship of the squared MMD with $Z_{\tilde{\pi}}$ makes the former infeasible to directly estimate, let alone minimize, for a general, unnormalized target.

The kernel Stein discrepancy (KSD), proposed for assessing sample quality~\cite{LiLeJo16,ChStGr16,GoMa17}, is a kernelized semi-metric invariant to this normalization and is equivalent to the MMD associated with a $\pi$-induced Stein kernel. KSD minimization has been used in the design of quadrature rules~\cite{ChMaGoBrOa18,KoAuMaAb21}. By construction, however, first-order minimization of the KSD requires second-order information from the target density $\pi$~\cite{KoAuMaAb21}, and thus it is often difficult or impossible to perform this minimization without resorting to greedy methods~\cite{ChMaGoBrOa18} or otherwise using methods that scale poorly in dimension.

%% file: sections/03_results.tex
\section{Main results}
As discussed above, the minimization of KL divergence (and related functionals) underlies many methods for sampling from unnormalized densities, in large part because the associated objectives admit gradient-based optimization schemes that are invariant to the unknown normalizing constant. In contrast, the minimization of MMD (excluding Stein kernels) is typically seen as infeasible here. We challenge this notion and show that
MMD minimization is achievable for unnormalized densities.

Recalling the definition of $\piquadYw$, consider the function $F^\pi:\RR^M\times\gX^M\to\RR^+$ given by
\begin{equation}\label{eq:F}
    F^\pi(\bm{w}, \Y):= \frac{1}{2} \mathrm{MMD}^2\left(\pi, \piquadYw \right).
\end{equation}
Observe that any critical point $(\wast,\Yast)\in\RR^{M} \times \gX^{M}$ of $F$ is defined via %
\begin{equation}\label{eq:grad_F_wrt_w_Y_0}
    \nabla_{\bm{w}} F^\pi (\wast, \Yast) = 0, \quad \nabla_{y_i} F^\pi (\wast, \Yast) = 0.
\end{equation}
We use conditions \eqref{eq:grad_F_wrt_w_Y_0}, which are necessary for optimality, as the desiderata for our quadrature rule. Recalling the notation $\vzero$ and $\vone$ from~\eqref{eq:vzero_def} and~\eqref{eq:vone_def}, introduce the squared exponential kernel $\kappa(x,y) = \exp\bigl( - \frac12 \sigma^{-2}\|x-y\|^2 \bigr)$ to express %
\eqref{eq:grad_F_wrt_w_Y_0} as
\begin{equation}\label{eq:critical_condition}
    \bm{K}(\Yast)\wast = \vzerob(\Yast),\quad
    \Wast \bigl( \bm{K}(\Yast) \bm{W}^\ast\ \Yast - \voneb(\Yast) \bigr) = \bm{0},
\end{equation}%
where $\Wast := \mathrm{diag}(\wast)$ is the diagonal weight matrix. For a given $\Y$, the first equation is linear in $\wast$; we thus let the function $\wast:\gX^M\to\RR^M$ describe the weights for configuration $Y$ as
\begin{equation}\label{eq:def_w_hat}
    \wY = \KY^{-1}\vzerob(\Y),
\end{equation}
and define $\WY = \mathrm{diag}(\wY)$.

The second condition in \eqref{eq:critical_condition} is more challenging. A strategy to satisfy the condition is to alternate between solving the linear system~\eqref{eq:def_w_hat} for $\wast$ and updating the configuration $\Y$ through the iteration
\begin{equation}\label{eq:msip_dynamics}
    \Ytp = (1-\eta)\Yt + \eta\Psi_{\MSIP}(\Yt).
\end{equation}
Here $\eta>0$ is a step size, and the function $\Psi_{\MSIP}:\gX^M\to\gX^M$ is defined by
\begin{equation}
    \Psi_{\MSIP}(\Y) := \WY^{-1}\KY^{-1}\voneb(\Y).
\end{equation}
This algorithm, known as \textit{mean shift interacting particles} (MSIP), was introduced in~\citep{BeShMa25} to minimize the MMD with respect to a probability measure $\pi$. In particular, it was proven that the dynamics \eqref{eq:msip_dynamics} correspond to a preconditioned gradient descent of the function $F^\pi_M:\gX^{M}\to\RR^+$ defined as
\begin{equation}\label{eq:F_M}
    F^\pi_M(\Y) = \inf_{\bm{w}\in\RR^M} F^\pi(\bm{w}, \Y) \, .
\end{equation}
The implementation of MSIP, however, was previously restricted to the ``data-driven'' setting where $\pi$ is available as an empirical distribution, as this input enables direct evaluation of the functions $v_{0}$, $v_{1}$.

In this work, we overcome this restriction and adapt MSIP to the case of an unnormalized density $\tilde{\pi}$.
Our approach builds on the observation that the MSIP map $\Psi_{\MSIP}$, unlike $F^\pi$ itself, is invariant to the normalizing constant. Unfortunately, this property alone does not suffice for a systematic implementation of MSIP in the unnormalized setting. Two more challenges must be addressed: numerical stability and reliance of the map $\Psi_{\MSIP}$ on the embeddings $v_{0}$ and $v_{1}$.%

\subsection{Regularized mean shift interacting particles}

We first grapple with extending the definition of the MSIP map to arbitrary particle configurations $Y$.
The problem of conditioning primarily stems from the explicit inversion of the kernel matrix $\bm{K}(Y)$. A classical remedy for such numerical issues is to replace $\bm{K}(Y)$ with its regularized counterpart $\bm{K}_\lambda(\Y) := \bm{K}(Y) + \lambda \bm{I}_M$ for some $\lambda >0$. 
We thus use a \textit{regularized} MSIP map $\Psi_{\MSIP,\lambda}$ defined as
\begin{equation}\label{eqn:def_msip_lambda}
    \Psi_{\MSIP,\lambda}(\Y) := \WLY^{-1}\bm{K}_\lambda(\Y)^{-1}\voneb(\Y)
\end{equation}
where $\WLY = \mathrm{diag}(\wlY)$, with 
\begin{equation}\label{eq:optimal_w_lam}
\wlY = \bm{K}_{\lambda}({Y})^{-1}\vzerob({\Y}),
\end{equation}
which is a regularization of the optimal weights defined by~\eqref{eq:def_w_hat}. This regularized map provides significantly more numerical stability---the matrix $\bm{K}_\lambda(\Y)$ is nonsingular for any configuration $\Y$---but, crucially, still yields preconditioned gradient descent of a penalized MMD, as we show below. Define the penalized objective function $F^\pi_{\lambda}: \RR^M\times\gX^M\to\RR^+$ as
\begin{equation}\label{eq:F_lambda}
    F^\pi_{\lambda}(\bm{w}, \Y)\coloneqq \frac{1}{2} \mathrm{MMD}\left(\pi, \piquadYw \right)^2 + \lambda \|\bm{w}\|^2 = F^\pi(\w,\Y) + \lambda \|\bm{w}\|^2,
\end{equation}
and consider the function $F^\pi_{M, \lambda}: \mathcal{X}^{M} \rightarrow \mathbb{R}_{+}$,
\begin{equation}\label{eq:F_lambda_M}
    F^\pi_{M,\lambda}(\Y) \coloneqq\inf_{\bm{w}\in\RR^M} F^\pi_{\lambda}(\bm{w}, \Y).
\end{equation}
\begin{theorem}\label{thm:lambda_MSIP_preconditioner}
For the squared exponential kernel, we have
    \begin{equation}
        \nabla F^\pi_{M,\lambda}(Y) = \sigma^{-2} \WLY\Big(  \bm{K}_{\lambda}(Y)  \WLY Y - \voneb(Y) \Big),
    \end{equation}
where $\bm{K}_{\lambda}(Y):= \bm{K}(Y) + \lambda \bm{I}_{M}$, and $\WLY=\textup{diag}(\wlY)$, with $\wlY$ from~\eqref{eq:optimal_w_lam}.
\end{theorem}
Due to~\Cref{thm:lambda_MSIP_preconditioner}, 
we can see that the regularized dynamics 
\begin{equation}\label{eq:lambda_msip_dynamics}
    \Ytp = (1-\eta)\Yt + \eta\Psi_{\MSIP, \lambda}(\Yt),
\end{equation}
for nonzero weights $\bm{\hat{w}}_\lambda$ leads to
\begin{align}
    \Ytp & = \Yt - \eta P_{\lambda}(\Y)^{-1}\nabla F^\pi_{M,\lambda}(Y) ,\\
    \text{with preconditioner } P_{\lambda}(\Y) & := \sigma^{-2} \WLY \bm{K}_{\lambda}(Y)\WLY.\label{eq:def_regularized_precond}
\end{align}
In other words, $\lambda$-regularization mollifies 
the target functional while ensuring that $\bm{K}_\lambda$ %
is invertible.

\subsection{Numerical implementations}\label{sec:vzero_vone_estimators}

Implementing the iteration~\eqref{eq:lambda_msip_dynamics} requires the kernel mean embedding $\vzerob$ and the kernelized first moment $\voneb$ for any set of particles $\Y$. Unlike the data-driven setting where $\pi$ is an empirical measure, these functions do not admit closed-form expressions. In what follows, we propose computationally tractable estimators when using the squared exponential kernel. First, note that
\begin{equation}\label{eq:primitive_v_pair}
\sigma^2\nabla\log v_0(y) = \frac{\sigma^2\nabla v_0(y)}{v_0(y)} = \frac{v_1(y) - yv_0(y)}{v_0(y)} = \frac{v_1}{v_0}(y) -y.
\end{equation}
Therefore, we can consider the primitive quantities to be the pair of functions $(v_0, v_1)$ or the pair $(\log v_0,\,\sigma^2\nabla\log v_0)$ with some equivalence.
Our estimators rely on the following observation:
\begin{align}
    \vzero(y) &= \int_{\mathbb{R}^d}\exp\left(-\frac{\|x-y\|^2}{2\sigma^2}\right)\pi(x)\mathrm{d}x = \omega_{\sigma,d}\int_{\mathbb{R}^{d}}\pi(y+\sigma\xi)\,\mathrm{d}\rho(\xi),
\end{align}
where $\rho$ is the standard multivariate Gaussian distribution and $\omega_{\sigma,d} := \big(\sqrt{2\pi}\sigma\big)^d$. Similarly, %
\begin{align}
    \vone(y) &= \omega_{\sigma,d} \int (y+\sigma\xi)\pi(y+\sigma\xi)\,\mathrm{d}\rho(\xi) \label{eq:vone_firstvariant} \\
    &= y\vzero(y) + \sigma\omega_{\sigma,d} \int \xi\pi(y+\sigma\xi)\,\mathrm{d}\rho(\xi). \label{eq:vone_secondvariant}
\end{align}

The problem of estimating $\vzero$ and $\vone$ over the unnormalized density $\pi$ thus reduces to computing integrals with respect to a multivariate Gaussian distribution. We propose to estimate these integrals via an inner %
quadrature rule. Suppose $(u_q,\xi_q)_{q=1}^Q\subset \RR\times\gX$ satisfies%
\begin{equation}\label{eq:inner_quadrature_rule}
\int \varphi(\xi) \mathrm{d}\rho(\xi) \approx \sum\limits_{q=1}^{Q} u_{q} \varphi(\xi_{q}),
\end{equation}
where $\varphi$ is a real-valued function defined on $\gX$. 
A first choice for this \textit{inner} quadrature rule is Monte Carlo, which involves sampling $\xi_1,\ldots,\xi_Q \stackrel{\mathrm{iid}}{\sim} \rho$ with weights $u_q=Q^{-1}$. A second choice uses spherical radial quadrature rules~\cite{BeZhMa24,ArHa09}. A third option is a simple one-point quadrature rule, i.e., using $Q=1$, with $u_{1} = 1$ and $\xi_{1}= 0$.
Using the discretization in~\eqref{eq:inner_quadrature_rule} for the integrand $\xi \mapsto \pi(y+\sigma \xi)$ is straightforward
and yields an approximation of $\vzero$, denoted by $\hvzero$ and defined by
\begin{equation}\label{eq:v_0_hat_quadrature}
    \hvzero^Q(y) := \omega_{\sigma,d} \sum\limits_{q = 1}^{Q} u_{q} \pi(y+ \sigma \xi_{q}).
\end{equation}
Since $\vone$ has multiple representations~\eqref{eq:vone_firstvariant}--\eqref{eq:vone_secondvariant}, strategies for approximating $\vone$ are discussed below.

\paragraph{Gradient-free estimator}

Applying the quadrature~\eqref{eq:inner_quadrature_rule} to the first integral representation of $\vone$~\eqref{eq:vone_firstvariant} leads to a gradient-free approximation, 
\begin{equation}
\hat{v}_1^{\mathrm{gf}}(y) := \omega_{\sigma,d}\sum\limits_{q=1}^Q u_q(y+\sigma \xi_q)\pi(y+\sigma\xi_q).
\end{equation}

\paragraph{Stein estimator}
Using Stein's identity~\cite{Ste81} simplifies the representation in~\eqref{eq:vone_secondvariant} of $\vone$:
\begin{equation}\label{eq:stein_identity}
    \vone(y) = y\,\vzero(y) + \sigma^2 \omega_{\sigma,d} \int_{\mathbb{R}^d} \nabla \pi(y+\sigma \xi) \mathrm{d}\rho(\xi).
\end{equation}
Applying the quadrature~\eqref{eq:inner_quadrature_rule} to~\eqref{eq:stein_identity}, we derive the Stein estimator:
\begin{equation}
    \hat{v}^{SQ}_{1}(y) = y\,\hat{v}^Q_{0}(y) + \sigma^2 \omega_{\sigma,d} \sum_{q=1}^{Q}(u_{q} \pi(y + \sigma\xi_q))\, \nabla \log\pi(y+\sigma \xi_{q}),
\end{equation}
noting that $\nabla\pi = \pi\nabla\log\pi$, where $\nabla\log\pi$ is often referred to as the (Stein) %
score.
For the one-point inner quadrature rule, the Stein estimator reduces to
\begin{equation}\label{eqn:one_point_defs}
    \hat{v}_0^1(y) = \omega_{\sigma,d}\pi(y),\quad \hat{v}_1^{\mathrm{S1}}(y) = \omega_{\sigma,d}\,\pi(y)\bigl(y+ \sigma^2  \nabla \log\pi(y) \bigr).
\end{equation}
In particular, we see $\sigma^2 \nabla \log \pi(y)$ corresponds to $\hat{v}_1^{\mathrm{S1}}(y) /\hat{v}_0^1(y)  -y$; this closely echoes $\nabla\log v_0$ in~\eqref{eq:primitive_v_pair}.
Moreover, the one-point estimators can recover a Gaussian deconvolution of $\pi$: see Proposition~\ref{prop:gmm_msip_stein} for details. For this reason, we consider $\hat{v}_0^1$ and $\hat{v}_0^{\mathrm{S1}}$ to be \textit{Fredholm} estimators.

\subsection{MSIP invariance to normalization}

Now that we have shown strategies for practically implementing MSIP, the regularized MSIP dynamics can be implemented using the following iteration:
\begin{equation}\label{eq:lambda_approx_msip_dynamics}
    \Ytp = (1-\eta)\Yt + \eta\hat{\Psi}_{\MSIP, \lambda}(\Yt),
\end{equation}
where $\hat{\Psi}_{\MSIP, \lambda}$ is obtained by replacing $\vzero$ and $\vone$ by $\hat{v}_{0}$ and $\hat{v}_{1}$ respectively in~\eqref{eqn:def_msip_lambda}.%
\begin{equation}
 \hat{\Psi}_{\MSIP, \lambda}(Y) = \bm{W}_{\lambda}(\Y)^{-1}\bm{K}_{\lambda}(\Y)^{-1}\hat{\bm{v}}_{1}(Y).
\end{equation}
We detail this method in \cref{alg:lambda_msip}. Next, we show that this iteration is invariant to the normalizing constant of a target density $\pi$.
 
\begin{proposition}\label{prop:msip_invariance}
Let $\Psi_{\mathrm{MSIP}, \lambda}(\cdot\,;\pi)$ be the regularized MSIP map for density $\pi$. For any $\Y\in\gX^M$,
\begin{equation}
     \Psi_{\mathrm{MSIP}, \lambda}(\Y;\pi) = \Psi_{\mathrm{MSIP}, \lambda}(\Y;\;C \, \pi) \ \ \forall C > 0.
\end{equation}
    
\end{proposition}

The invariance of $\hat{\Psi}_{\mathrm{MSIP},\lambda}$ to normalization reasons similarly to Proposition~\ref{prop:msip_invariance}. 
In essence, this invariance comes from the preconditioner $P_\lambda$ in~\eqref{eq:def_regularized_precond}. While it may seem at first glance that division by the weights $\wast_\lambda$ echoes self-normalization in, e.g., importance sampling or CBS, such parallels are largely superficial. Whereas the drift term for any particle $\yi$ in CBS is nearly affine in its assigned weight $w_i$,%
\footnote{The CBS drift is affine in the weights $\alpha\wi$ for $\alpha = 1/\sum \wi$.
Specifically, the Jacobian of the drift w.r.t.~$\bm{w}$ has $i$-th column $\alpha\bm{e}_i + \alpha^2 \bm{w}$. This approaches an inverse relationship for one weight $\wi$ if $\alpha w_i \approx 1$, but remains nearly affine in all other weights.%
}%
the dynamics of $\yi$ in MSIP are \textit{inversely related} to the kernel-optimal weight $\whati$.

\subsection{Related kernelized methods}
With details of our method in hand, we contrast it with other recent kernel methods used in sampling and optimization. 
One line of work 
considers interacting particle methods for Bayesian neural networks.
In these approaches, repulsion and diversity are enforced at the level of network predictions rather than directly on the weights; see, e.g.,~\cite{DaFo21}. Understanding the precise relationship between such function-space interactions and parameter-space interacting particle systems such as MSIP lies outside the scope of the present work.
Polarized CBS~\cite{BuRoWa25} enriches classical CBS with interaction mechanisms designed to improve exploration of multi-modal distributions. While polarized CBS shares certain similarities with our gradient-free MSIPGF, the two approaches arise from fundamentally different principles: the former remains centered around consensus-driven dynamics, whereas MSIPGF constructs particle interactions through kernelized approximations of nonlocal quantities such as $v_0$ and $v_1$. The theoretical analysis in~\cite{BuRoWa25} is primarily carried out in the large-$M$ regime.
Drifting~\cite{DeLiLiDuHe26} is a recent kernelized approach for generative modeling that constructs a vector field from the discrepancy between mean-shift maps associated with the target samples and the particle distribution. In contrast to MSIP, drifting is designed for sample-based generative modeling rather than for sampling from unnormalized densities.

%% file: sections/04_experiments.tex
\begin{figure}[h!]
\centering
    \includegraphics[width=\linewidth]{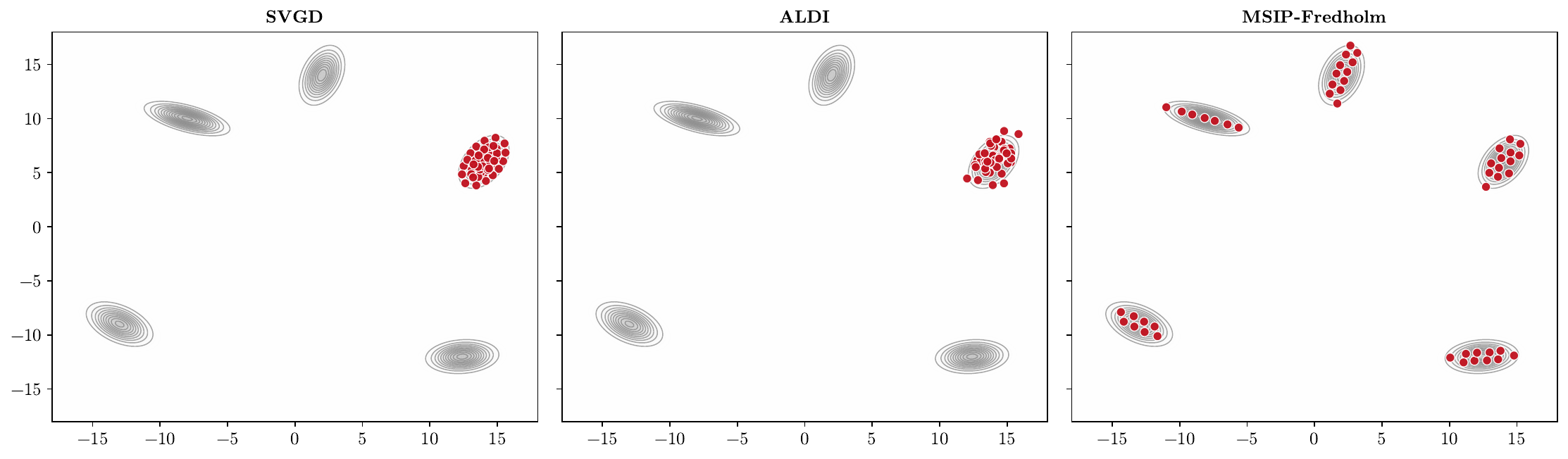}

\caption{
A comparison of our methods to other samplers on a mixture of five anisotropic Gaussians.
}
\label{fig:MSIP_for_GMMs_qual}
\end{figure}

\section{Experiments}
We now evaluate the proposed algorithms against established methods, including SVGD, ALDI, and CBS approaches. We use both synthetic distributions and real-world Bayesian posteriors---for error quantification, we normalize any weights to ensure unit sum.%
\footnote{All code used to generate the experiments is available at~\url{https://github.com/Nodes-and-Kernels/nak_torch}.}

\subsection{Mixtures of Gaussians}

We use Gaussian mixtures to assess each method's ability to capture multi-modality, anisotropy and separation between modes. Further, for mixtures of Gaussians, one can evaluate the MMD analytically, which we use to assess the ability of the proposed algorithm in MMD minimization.

In the first simulation, we evaluate the proposed methods on a multi-modal target distribution defined by a five-component anisotropic Gaussian mixture in $\mathbb{R}^2$. We compare SVGD and gradient-informed ALDI to MSIP-Fredholm. For each number of particles $M$, we run multiple independent initializations from a Gaussian $\mathcal{N}(\mu, \mathbb{I}_{2})$ and evolve the particles for a fixed number of iterations $T$, where  $\mu = (18, 18)$.
\Cref{fig:MSIP_for_GMMs_qual} shows the configuration corresponding to the last iteration of each of the three algorithms. We observe that MSIP-Fredholm captures all modes of the distribution, while SVGD and ALDI suffer from mode collapse. This qualitative observation is validated by the quantitative comparison given in \Cref{fig:algos_2dgmms_mmd}, where the MMD at the last iteration of each algorithm is shown as a function of the particle count. We see MSIP outperforms SVGD and ALDI by a significant margin, especially when the number of particles is large.

\begin{figure}[h]
    \centering
\includegraphics[width=0.5\linewidth]{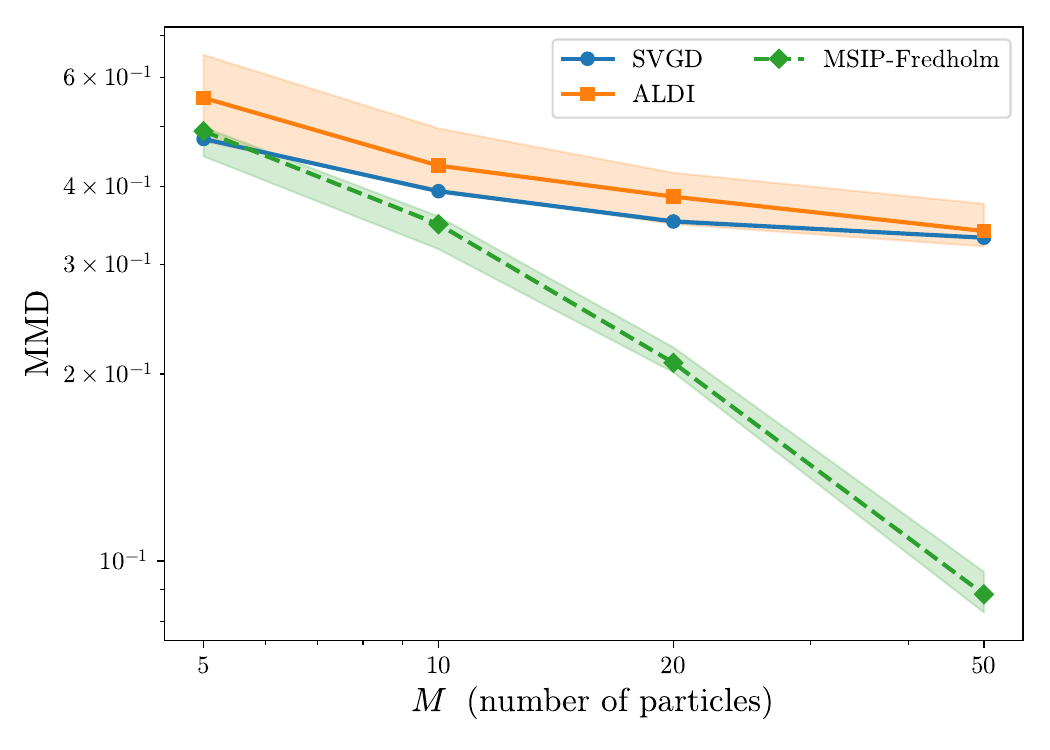}
    \caption{MMD decay with respect to the number of particles $M$ for three algorithms, over 20 independent trials. The shaded regions represent the 90\% confidence intervals over the trials.}
    \label{fig:algos_2dgmms_mmd}
\end{figure}

\subsection{Benchmark sampling problems}\label{sec:distribution_experiments}
In the second batch of experiments, we evaluate MSIP on a broader collection of target distributions. In particular, we compare several implementations of MSIP with state-of-the-art algorithms such as SVGD, ALDI, and CBS. More precisely, we consider three gradient-informed implementations defined in \Cref{sec:vzero_vone_estimators}:
i) MSIP-F, which is based on the Fredholm estimators~\eqref{eqn:one_point_defs}; %
ii) MSIP GI-1, which is based on the Stein estimator, where the inner quadrature rule \eqref{eq:inner_quadrature_rule} corresponds to Monte Carlo sampling from the Gaussian distribution $\mathcal{N}(0, \sigma \mathbb{I}_d)$ with $Q = 1$;
iii) MSIP-GI-10, which is also based on the Stein estimator, with the same inner quadrature rule but using $Q = 10$.
We also consider a gradient-free implementation defined in \Cref{sec:vzero_vone_estimators}: MSIP-GF, which is based on the estimator $\hat{v}_1^{\mathrm{gf}}$, where the inner quadrature rule \eqref{eq:inner_quadrature_rule} corresponds to Monte Carlo sampling from the Gaussian distribution $\mathcal{N}(0, \sigma \mathbb{I}_d)$ with $Q = 10$. \Cref{tab:ksd_synthetic} reports the resulting kernel Stein discrepancy (KSD) values on a collection of synthetic distributions with varying anisotropies and multi-modal behavior.

\begin{figure}[h]
    \centering
    \includegraphics[width=0.55\linewidth, clip, trim={0 0 0 1.3cm}]{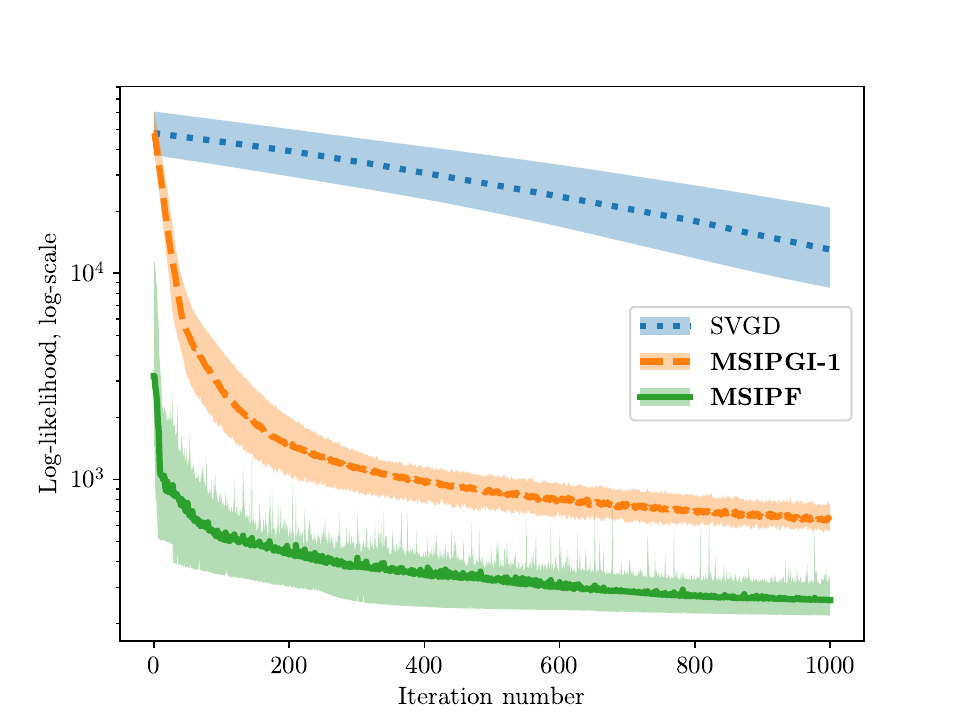}
    \caption{
    Log-likelihood of the neural network ensemble in the TwoMoons example over 20 independent trials. The shaded regions represent the 90\% confidence interval as a function of the iteration number.}
    \label{fig:convergence_bnn}
\end{figure}

\begin{table}[ht]
{
    \caption{Error of different algorithms for synthetic distributions using KSD with an inverse multi-quadric kernel. We report mean and standard deviation over many trials; the smallest row-wise mean is bolded. Empty cells correspond to an algorithm not converging.}
    \tiny
    \centering
    \label{tab:ksd_synthetic}
    \begin{tabular}{@{}l c c c c c c c c@{}}\toprule
    & & \multicolumn{5}{c}{Gradient-informed algorithms} & \multicolumn{2}{c}{Gradient-free algorithms} \\
    \cmidrule(lr){3-7}\cmidrule(lr){8-9}
     & Dim. & \textbf{MSIPF} & \textbf{MSIPGI-1} & \textbf{MSIPGI-10} & SVGD & ALDI & \textbf{MSIPGF} & CBS \\\midrule
    GMM2 & 2 & 1.000 (0.056) & 1.489 (0.499) & 3.126 (5.286) & 1.223 (0.039) & \textbf{0.930 (0.094)} & 6.014 (5.028) & 1.103 (0.090) \\
    GMM5 & 5 & 1.510 (0.122) & 3.495 (4.684) & \textbf{1.375 (0.022)} & 1.499 (0.004) & 1.516 (0.054) & 6.916 (12.33) & 3.293 (0.336) \\
    GMM10 & 10 & 2.229 (0.002) & \textbf{1.970 (0.010)} & 2.041 (0.039) & 2.172 (0.002) & 2.185 (0.039) & 3.777 (2.145) & 3.899 (0.369) \\
    GMM20 & 20 & 3.395 (1.531e-05) & \textbf{3.012 (0.030)} & 3.014 (0.040) & 3.060 (0.003) & 3.355 (0.171) & 3.110 (0.074) & 5.061 (0.345) \\
    Joker & 2 & 10.26 (0.474) & \textbf{8.493 (0.016)} & 8.563 (0.068) & 11.11 (1.579) & 9.286 (0.185) & 8.732 (0.369) & 10.02 (1.475) \\
    Funnel2 & 2 & 9.104 (0.085) & \textbf{8.618 (0.022)} & 8.620 (0.017) & 9.073 (0.145) & 9.008 (0.091) & 8.763 (0.019) & 95.40 (176.8) \\
    Funnel5 & 5 & 30.44 (5.118) & \textbf{14.25 (0.021)} & 14.31 (0.029) & 18.64 (2.574) & 14.43 (0.210) & 14.30 (0.021) & 14.98 (0.615) \\
    Funnel10 & 10 & 83.79 (17.37) & \textbf{20.24 (0.039)} & 20.48 (0.033) & 24.06 (1.593) & 20.34 (0.167) & 20.40 (0.075) & 21.86 (1.396) \\
    Funnel20 & 20 & 131.0 (16.90) & 28.92 (0.160) & 29.97 (0.408) & 67.36 (22.99) & \textbf{28.66 (0.122)} & 29.86 (0.377) & 34.05 (3.475) \\
    Funnel50 & 50 & 236.4 (0.617) & \textbf{46.68 (0.687)} & 49.75 (1.725) & 242.5 (142.7) & --- & 47.64 (0.887) & 58.30 (10.52) \\
    Himmelblau & 2 & 8.129 (0.182) & 9.778 (1.952) & \textbf{7.671 (0.188)} & 14.24 (0.311) & 14.33 (2.558) & 8.071 (0.335) & --- \\
    \bottomrule
    \end{tabular}
}
\end{table}

\begin{table}[ht]
    \centering
    \tiny
    \caption{Error on benchmark problems in log-likelihood (LL) and kernel Stein discrepancy (KSD). Examples are suffixed by the dimension $d$ of the target. An empty cell denotes failure of the algorithm.}
    \label{tab:ksd_ll_benchmarks}
    
    \begin{tabular}{@{}l l c c c c c c@{}}\toprule
        & & \textbf{MSIPF} & \textbf{MSIPGI-1} & SVGD & ALDI & \textbf{MSIPGF} & CBS \\\midrule
    \multirow{5}{*}{LL} &
    Schools10 & \textbf{6.553 (1.63)} & 12.76 (1.04) & 11.89 (2.67) & 9.152 (0.800) & 21.68 (11.2) & 9.823 (2.43)\\
    & COVID51 & \textbf{4616 (409)} & 6782 (155) & 471239 (37239) & --- & 780936 (42414) & 655235 (68207)\\
    & CovType56 & 293827 (68258) & 645767 (59017) & \textbf{287376 (1911)} & --- & 332276 (11689) & ---\\
    & ABPDE64 & \textbf{546 (140)} & 1062 (563) & 3203 (4292) & --- & 5447 (6977) & ---\\
    & TwoMoons201 & \textbf{258.60 (27.07)} & 638.76 (50.87) & 13653 (3594) & --- & 45670 (6570) & ---\\
    \midrule

    \multirow{5}{*}{KSD} &
    Schools10 & 28.7 (7.5) & 9.6 (0.18) & 9.50 (1.1) & \textbf{9.04 (0.04)} & 20.2 (18) & 10.6 (1.8)\\

    & COVID51 & 507 (74) & \textbf{106 (1.7)} & 37817 (1028) & --- & 42176 (850) & 45396 (1393)\\

    & CovType56 & 35140 (21703) & 15373 (1216) & \textbf{1865.2 (433)} & --- & 3667 (1235) & ---\\

    & ABPDE64 & 164 (20) & \textbf{42.3 (2.3)} & 818 (1782) & --- & 4293 (9609) & ---\\

    & TwoMoons201 & 267 (50) & \textbf{168 (15)} & 449 (76) & --- & 801 (45) & ---\\

    \bottomrule
    \end{tabular}
\end{table}

We observe that the different variants of MSIP consistently achieve either the lowest KSD values or values extremely close to the best-performing method across essentially all benchmark distributions. In contrast, competing approaches exhibit substantially less robustness across geometries and dimensions. In particular, SVGD frequently underperforms relative to MSIP variants, not only on challenging targets such as funnel, joker, and Himmelblau distributions, but also on several Gaussian mixture models, especially when the evaluation bandwidth matches the characteristic scale used to construct the targets. These observations also highlight a limitation of KSD-based evaluation: on multi-modal distributions, small KSD values do not necessarily imply accurate global exploration or faithful mode allocation. This phenomenon appears sharply when examining visualizations of the two-dimensional distributions, as in \cref{sec:two_dim_viz}. Similarly, ALDI performs strongly on funnel distributions, where its affine-invariant structure is advantageous, but is generally less competitive on the remaining targets. 
By contrast, the gradient-free variant MSIP-GF remains remarkably competitive across nearly all synthetic experiments, including difficult high-dimensional settings. Overall, the MSIP family is the only class of methods in this benchmark that consistently maintains strong and stable performance simultaneously across multi-modal, anisotropic, heavy-tailed, and high-dimensional targets.

\cref{tab:ksd_ll_benchmarks} presents a broad evaluation of our methods on non-synthetic target distributions, where any large deviation suggests near-failure in a given trial.\footnote{For each (weighted) point set and unnormalized target density $\widetilde{\pi}$, we calculate the log-likelihood as $\mathcal{L}(\Y,\w;\widetilde{\pi})\coloneqq-\sum_{k=1}^M w_k\log\widetilde{\pi}(y_k)$; this can be negative for unnormalized $\widetilde{\pi}$.}
Remarkably, virtually all other methods often fail without careful tuning of step size across problems; the ALDI results, in particular, show the peril of algorithmic sensitivity to step size. As a case study, \Cref{fig:convergence_bnn} demonstrates the convergence in log-likelihood of MSIP for the Bayesian neural network example of classification in a problem with $d=201$ parameters; we see that both MSIP-Fredholm and MSIP-GI converge much quicker than the SVGD results, showing the strength of our preconditioned dynamics.

Note that the costs of MSIP-Fredholm and the one-point MSIP GI-1 are both \emph{equivalent} in number of density evaluations to SVGD. On the other hand, the inner quadrature rules used for MSIP GI-10 and MSIP GF both require more information from the target $\pi$. The non-local information these estimators glean comes from a user-specified budget; the corresponding evaluations, however, can be performed in parallel. Indeed, vectorization and mapping operations common in machine learning toolkits (e.g., the \texttt{vmap} transformation in torch/JAX) yield virtually identical cost to MSIP-Fredholm for small quadrature rules, depending on the problem dimension $d$ and the number of particles $M$.

In \Cref{sec:two_dim_viz}, we provide several two-dimensional case studies comparing our contributed algorithms to many different other methods in final configurations for comparable computational cost. In these two-dimensional examples, we show consistently high-quality quadrature rules produced by MSIP methods. Further, we see that MSIP methods produce well-spaced point sets that go beyond the typical independent sampling approaches considered by most other versatile point generation approaches.

%% file: sections/05_conclusions.tex
\section{Conclusion}
We have shown that MSIP is an effective new approach for producing \emph{weighted quadratures} of distributions known only through their unnormalized density. 
 {To our knowledge, it is the most scalable method for MMD minimization in this setting.}
Our experiments demonstrate that MSIP algorithms enjoy robust performance over a very broad range of target distributions, achieving approximation quality that is consistently competitive with---and often superior to---current state-of-the-art sampling approaches. 
A key strength of MSIP is that its dynamics balance systematic modal coverage (i.e., quickly finding all modes of a target) with accurate representation of the shape and extent of each mode. 

Underlying MSIP is the internal quadrature rule, a distinctive degree of freedom not afforded in typical approaches. Through the quantities $v_0$ and $v_1$, this quadrature highlights non-local information, in lieu of dynamics driven by the score function alone. Our experiments suggest that accurately approximating these global kernel-based quantities can lead to substantial improvements in performance, particularly for challenging multi-modal and anisotropic targets.

Future work will focus on establishing theoretical guarantees for MSIP, including convergence and approximation results, as the nonconvex optimization at the core of the method remains challenging analytically. We also plan to advance systematic and adaptive strategies for selecting the internal quadrature rule and associated kernel hyperparameters.

\paragraph{Acknowledgments}

DS acknowledges support from a 2025--2026 MathWorks Fellowship. AB and YMM acknowledge support from the ExxonMobil Technology and Engineering Company. DS and YMM acknowledge support from the US Department of Energy (DOE), Office of Science, Office of Advanced Scientific Computing Research (ASCR), via the M2dt MMICC center under award number DE-SC0023187 and through the Scientific Discovery through Advanced Computing (SciDAC) FASTMath Institute, under contract number DE-AC52-07NA27344. AB and YMM also acknowledge support from DOE ASCR under award numbers DE-SC0023188 and DE-SC0026245.

\newpage

%% file: sections/appendix.tex
\section{Numerics}
\subsection{The algorithm}

\begin{algorithm}[H]
\caption{Regularized MSIP \label{alg:lambda_msip}}
\KwIn{Density $\pi$, step size $\eta$, number of iterations $T$, regularization parameter $\lambda$, hybridization rate $\gamma$, inner quadrature rule $(u_q,\xi_q)_{q=1}^Q\subset\RR\times\gX$, initialization $Y^{(0)} = \{y_1^{(0)}, \dots, y_M^{(0)}\}$}
\KwOut{Last configuration $Y^{(T)}$}

\For{$t = 0, 1, \dots, T-1$}{

    \For{$i =  1, \dots, M$}{
    $\hat{v}_{0}(y_i) \gets \omega_{\sigma,d} \sum_{q = 1}^{Q} u_{q} \pi(y_i+ \sigma \xi_{q})$ \;

    $\hat{v}_{1}^{H}(y_i) \gets  (1-\gamma) \hat{v}_{1}^{\mathrm{gf}}(y) + \gamma \hat{v}_{1}^{S}(y)$\;

    }

    $\bm{K}_{\lambda}(Y) \gets \bm{K}(Y) + \lambda \bm{I}_{M}$ \;

    $\hat{\bm{w}}_{\lambda}(Y) \gets \bm{K}_{\lambda}({Y})^{-1}\hat{\bm{v}}_{0}(Y) $\;

    $\WLY \gets \mathrm{diag}(\wlY)$ \;

    $\Psi^{(t)} \gets (\bm{K}_\lambda(\Y)\WLY)^{-1} \hat{\bm{v}}_1^{H}(\Y)$\;

    $Y^{(t+1)} \gets (1-\eta) \, Y^{(t)} + \eta \Psi^{(t)}$\;%

}

\end{algorithm}

\subsection{MSIP for mixtures of Gaussians}

Consider a mixture of Gaussians
\begin{equation}
\pi(x) = \sum_{k=1}^{K} m_k \mathcal{N}(x;\, \mu_k,\, \Sigma_k)
\end{equation}
with
\begin{equation}
\mathcal{N}(x ;\,\mu_k, \Sigma_k) := \frac{1}{(2\pi)^{d/2} |\Sigma_k|^{1/2}} \exp\left( -\frac{1}{2} (x - \mu_k)^T \Sigma_k^{-1} (x - \mu_k) \right)
\end{equation}
We consider the kernel $\kappa_\sigma(x, y) = \exp(-\frac{1}{2\sigma^2}\|x-y\|^2) = Z_{\sigma} \mathcal{N}(x;\, y, \sigma^2 I)$ for $Z_\sigma = (2\pi\sigma^2)^{d/2}$, and we obtain
\begin{align}
v_0(y) = \int \kappa_\sigma(x, y) \pi(x) \mathrm{d}x= & Z_{\sigma}\sum_{k=1}^{K} m_k \mathcal{N}(y \mid \mu_k, \widetilde{\Sigma}_k)\\ = & \frac{Z_{\sigma}}{(2\pi)^{d/2} }\sum_{k=1}^{K} \frac{m_k}{|\widetilde{\Sigma}_k|^{1/2}} \exp\left( -\frac{1}{2} (y - \mu_k)^T \widetilde{\Sigma}_k^{-1} (y - \mu_k) \right)
\end{align}
where $\widetilde{\Sigma}_k = \Sigma_k + \sigma^2 I$.
Now, let $\widetilde{\Sigma}_k = L_k L_k^\top$ be the Cholesky decomposition of $\widetilde{\Sigma}_k$ and let $z_k := L_k^{-1}(y - \mu_k)$. We have $\log |\widetilde{\Sigma}_k|^{1/2} = \sum_{j=1}^d \log [L_k]_{jj}$ and, thus,
\begin{equation}
\log v_0(y) = \log \sum_{k=1}^{K} \exp\left( \log m_k - \sum_{j=1}^d \log [L_k]_{jj} - \frac{1}{2} \|z_k\|_2^2 \right) + d\log \sigma
\end{equation}
Therefore,
\begin{equation}
\nabla \log v_0(y) = \sum_{k=1}^{K} r_k(y) \left( - \widetilde{\Sigma}_k^{-1} (y - \mu_k) \right),
\end{equation}
where
\begin{equation}
r_k(y) := \frac{m_k\, \mathcal{N}(y ;\, \mu_k,\, \widetilde{\Sigma}_k)}{\sum_{i=1}^{K} m_i\, \mathcal{N}(y ;\, \mu_i,\, \widetilde{\Sigma}_j)}
\end{equation}
Using the Cholesky factors, $\widetilde{\Sigma}_k^{-1}(y - \mu_k) = L_k^{-\top} z_k$, we get
\begin{equation}
\nabla \log v_0(y) = - \sum_{k=1}^{K} r_k(y) L_k^{-\top} z_k.
\end{equation}
This last identity is purely for computational convenience.

\subsection{Experimental details}\label{sec:hyperparams}
We now discuss the details of the experiments provided in~\Cref{sec:distribution_experiments}. All experiments were performed on standard consumer-grade CPUs, except for ABPDE, which was performed using an RTX4090. All code was implemented in torch.

\paragraph{GMM} The datasets prefixed by GMM are chosen Gaussian mixture models in arbitrary dimension, where GMM-$d$ has components $\mathcal{N}(\cdot;\mu_k,\sigma^2 I_d)$ for $k=1,\ldots,5$ with $\mu_k = 7.5$ and $\sigma^2 = 0.5$.

\paragraph{Joker} The Joker distribution is a mixture of two Gaussian densities of different anisotropy and a ``banana'' distribution that is \textit{not sub-Gaussian}---this is a preliminary test of robustness to non-Gaussianity. See \cref{fig:2dvis_joker} for a visualization.

\paragraph{Funnel} The funnel distributions, originally from~\cite{Nea03}, are classic test cases used in benchmarking sampling algorithms in the presence of high anisotropy; these have been used to benchmark methods in, e.g.,~\cite{MaMa24}. Up to a constant, the density is expressed as:
\[\pi(x_1,x_{2:d}) \propto \mathcal{N}(x_1;\mu,\sigma^2 I)\mathcal{N}\big(x_{2:d};0,\exp(x_1)I\big).\]

\paragraph{Himmelblau} Similarly, the Himmelblau function~\cite{Him72}, with unnormalized density
\[
\pi(x)\propto \exp\left(-(x_1^2 + x_2 - 11)^2 - (x_1 + x_2^2 - 7)^2\right)
\]
is a two-dimensional test case common for checking the behavior of sampling methods~\cite{IgCi23}.

\paragraph{Schools} We use the ``eight schools'' posterior, implemented in \texttt{posteriordb}~\cite{MaBuVe23}, as described in~\cite{GeHi06}.

\paragraph{COVID} Similar to the above, we use a posterior from a report describing the spread and contagion of COVID-19~\cite{FlMiGa20}, as implemented in \texttt{posteriordb}.

\paragraph{CovType} We perform hierarchical Bayesian logistic regression on a large dataset with an identical setup to~\cite{LiWa16}. That is, mini-batching 50 random data samples out of the entire dataset at each step of the algorithm.

\paragraph{ABPDE} We use a PDE-constrained inverse problem
intended as a benchmark for Bayesian algorithms, as proposed by Aristoff and Bangerth~\cite{ArBa23}. In brief, the Bayesian problem here is inference of piecewise-constant material properties on a two-dimensional square domain, in a linear elasticity PDE. The map from material properties to the PDE solution (strain/displacement) is nonlinear, so the posterior is intrinsically non-Gaussian. Non-Gaussianity is amplified by treating the material properties on a logarithmic scale (to enforce positivity).
To avoid an inverse crime, we add Gaussian noise of the appropriate amount to the proposed optimum.

\paragraph{TwoMoons} We consider the optimization of Bayesian neural networks for a two-dimensional classification task. The network is a simple multilayer perception with width 50 and one hidden layer. We use mini batches of size 64 at each step of each algorithm.

We provide several cross-algorithm hyperparameters in~\cref{app_tab:hyperparams}. For each example, we take 25 particles in every method, CBS uses inverse temperature $\beta=0.9$, and we use ten Monte Carlo samples for the inner quadrature rule of MSIPGF as well as MSIPGI-10.

\begin{table}[ht!]
    \centering
    \caption{Algorithm hyperparameters}
    \label{app_tab:hyperparams}{\tiny
    \begin{tabular}{@{}l c c c c c c c c@{}}\toprule
        Example & Dimension & Step Size & MSIP Bandwidth & KSD Bandwidth & Steps & Trials & Bounds\\\midrule
        GMM5 & 5 & $0.5$ & $0.5$ & $0.5$ & $1000$ & 10 & $(-10^3, 10^3)$\\
        GMM10 & 10 & $0.5$ & $0.5$ & $0.5$ & $1000$ & 10 & $(-10^3, 10^3)$\\
        GMM20 & 20 & $0.5$ & $0.5$ & $0.5$ & $1000$ & 10 & $(-10^3, 10^3)$\\
        Joker& 2 &  $0.5$ & $0.1$ & $0.1$ & $1000$ & 10 & $(-10^3, 10^3)$\\
        Funnel2 & 2 & $0.5$ & $0.1$ & $0.1$ & $1000$ & 10 & $(-10^3, 10^3)$\\
        Funnel5 & 5 & $0.05$ & $0.1$ & $0.1$ & $1000$ & 10 & $(-10^3, 10^3)$\\
        Funnel10 & 10 & $0.05$ & $0.1$ & $0.1$ & $1000$ & 10 & $(-10^3, 10^3)$\\
        Funnel20 & 20 & $0.05$ & $0.1$ & $0.1$ & $1000$ & 10 & $(-10^3, 10^3)$\\
        Himmelblau & 2 & $0.1$ & $0.05$ & $0.1$ & $1000$ & 10 & $(-10^3, 10^3)$\\
        Schools & 10 & $10^{-3}$ & 0.1 & 0.1 & 1000 & 20 & $(-10^3, 10^3)$\\
        COVID & 51 & $10^{-3}$ & 0.1 & 0.1 & 500 & 10 & $(-10^3, 10^3)$\\
        CovType & 56 & $10^{-3}$ & 0.1 & 0.1 & 1000 & 20 & $(-10^3, 10^3)$\\
        ABPDE & 64 & $10^{-3}$ & 0.1 & 0.1 & 1000 & 20 & $(-10^3, 10^3)$\\
        TwoMoons & 201 & $10^{-3}$ & 0.1 & 0.1 & 1000 & 20 & $(-10^3, 10^3)$\\

        \bottomrule
    \end{tabular}}
\end{table}

\section{Proofs}\label{sec:proofs}

\subsection[Proof of MSIP Preconditioner]{Proof of~\Cref{thm:lambda_MSIP_preconditioner}}

First, let $C_\pi := \int_{\X} \int_{\X} \kappa(x,x^{\prime}) \dpix \mathrm{d}\pi(x^{\prime})$. Then, observe that for $\Y\in \X^{M}$ we have
\begin{equation}
    F_{M, \lambda}(\Y) = \frac{1}{2} \bigg( C_\pi - 2 \langle \what, \vzerov(\Y) \rangle + \langle \what, \Kyl \what \rangle  \bigg),
\end{equation}
where $\hat{\bm{w}}(\Y)$ is defined by \eqref{eq:optimal_w_lam} (omitting the $\lambda$ subscript for concision), and $\vzero$ is given by \eqref{eq:vzero_def}. Using \eqref{eq:optimal_w_lam}, we see that
\begin{equation}\label{eq:kernel_ids}
    \langle \hat{\bm{w}}(\Y), \vzerov(\Y) \rangle = \langle \vzerov(\Y), \Kyl^{-1} \vzerov(\Y) \rangle = \langle \hat{\bm{w}}(\Y), \Kyl \hat{\bm{w}}(\Y) \rangle.
\end{equation}
In particular, the expression of $F_{M,\lambda}(\Y)$ simplifies further to
\begin{align*}
    F_{M, \lambda}(\Y) & = \frac{1}{2} \bigg(C_\pi - \langle \what, \vzerov(\Y) \rangle \bigg) = \frac{1}{2} \bigg( C_\pi - \langle \what, \Kyl \what \rangle \bigg).
\end{align*}

In the following, we proceed to the calculation of the gradient of $F_{M}$. First, for a configuration $Y\in\X^M$, we define $y_{ij}$ as the scalar in the $j$th dimension of node $\yi$. Now, we seek the gradient $\nabla F_{M}\in\RR^{M\times d}$. In particular, for $i \in [M]$ and $j \in [d]$, we use $\nabla_{ij} F_{M}(\Y)$ to denote the $(i,j)$ entry of $\nabla F_{M}(Y)$ corresponding to input $\yij$ evaluated at $\Y\in\mathbb{R}^{M \times d}$.
Moreover, we define $\nabla_{ij} \Kyl\in\RR^{M\times M}$  to be the differentiation of each element of matrix $\Kyl(\Y)$ with respect to the $j$th coordinate of vector $\yi$. Similarly, for kernel mean embedding $\vzerov: \mathcal{X}^{M} \rightarrow \mathbb{R}^{M}$ we denote by $\nabla_{ij} \vzerov$ the vector obtained by differentiating each element of $\vzerov$ with respect to the $j$th coordinate of vector $\yi$. In other words, we have
\begin{equation}
    \Dij F_{M, \lambda} = \frac{\partial}{\partial \yij} F_{M, \lambda}(\Y),\quad \Dij \Kyl = \frac{\partial}{\partial \yij} \Kyl , \quad \Dij \bm{v}(\Y) = \frac{\partial}{\partial \yij} \bm{v}(\Y)
\end{equation}

Using \eqref{eq:kernel_ids} and matrix calculus, we have
\begin{equation}\label{eq:gradient_optimal_mmd_identity}
    2\Dij F_{M, \lambda}(\Y) = - 2\langle \what,\,\Kyl\Dij\what\rangle - \langle \what,\,\Dij\Kyl\,\what \rangle .
\end{equation}
Since $\what = \Kyl^{-1}\vzerov(\Y) $, we get
\begin{equation*}
\Dij \what = \Kyl^{-1}\Dij \vzerov(\Y) -\Kyl^{-1}\,\Dij\Kyl\,\Kyl^{-1}\vzerov(\Y),
\end{equation*}
thus
\begin{equation*}
    \Kyl\Dij\what = \Dij\vzerov(\Y) - \Dij\Kyl\,\what,
\end{equation*}
so that
\begin{equation}
    \langle\what,\ \Kyl \nabla_{ij} \what\rangle = \langle\what,\ \nabla_{ij} \vzerov(\Y)\rangle -\langle \what,\ \nabla_{ij}\Kyl \what\rangle.
\end{equation}
Therefore, we have
\begin{gather}\label{eq:grad_F_pi_first}
\begin{split}
2\nabla_{ij} F_{M, \lambda}(\Y) &= -\,2\Big(\,\langle\what,\, \Dij\bm{v}_0(\Y)\rangle - \langle\what,\,\Dij\Kyl\,\what\rangle\,\Big) \\
& \quad \: - \langle \what,\, \Dij\Kyl\,\what\rangle \\
&= -2\langle \what,\,\Dij \bm{v}_0(\Y)\rangle + \langle\what,\,\Dij\Kyl\,\what\rangle.
\end{split}
\end{gather}
Now recall that %
\begin{equation*}
    [\Dij\Kyl]_{m_1,m_2} = \Dij\kappa(y_{m_1}, y_{m_2}),
\end{equation*}
for all $m_1,\,m_2 \in [M]$; the regularization parameter $\lambda$ disappears here as it is not a function of $Y$. Therefore, for the squared-exponential $\kappa$ with $\bar{\kappa}(x,y) := \sigma^{-2}\kappa(x,y)$,
\begin{align}
\nabla_{ij} \kappa(y_{m_1},y_{m_2}) & = \delta_{i,m_1}(y_{m_2j} - \yij)\bar{\kappa}(\yi,y_{m_2}) + \delta_{i,m_2}(y_{m_1j} - \yij)\bar{\kappa}(y_{m_1}, \yi) \\
& = \sigma^{-2} \Big( \delta_{i,m_1}(y_{m_2j} - \yij)\kappa(\yi,y_{m_2}) + \delta_{i,m_2}(y_{m_1j} - \yij)\kappa(y_{m_1}, \yi)\Big)
\end{align}

First, we denote $[\bm{a}\otimes\bm{b}]_{ij} = a_ib_j$ as the outer product between vectors $\bm{a}$ and $\bm{b}$ with possibly different dimensions. We also denote $\bm{a}\odot\bm{b}$ as the elementwise (i.e., Hadamard) product between $\bm{a}$ and $\bm{b}$ with identical dimensions. Then, we have
\begin{equation*}
        \nabla_{ij} \Kyl = \bij\otimes \bm{e}_i + \bm{e}_i\otimes \bij,
\end{equation*}
where $\bm{e}_i$ is the $i$th elementary vector, i.e., $[\bm{e}_i]_{m} = \delta_{im},$ and we use the expression of the squared exponential kernel
to get definition
\begin{equation*}
        \bij := \sigma^{-2}(\Y_j - \yij\bm{1})\odot \Ki = \sigma^{-2}\Big(\Y_{j}\odot \Ki - \yij \Ki\Big),
\end{equation*}
with $\Y_{j} \in \RR^{M}$ is the vector containing the $j$-th entry of each $\yi$, identical to the $j$th column of the matrix $\Y$, and $\Ki = \Km(\Y)\bm{e}_i$. Then,
\begin{align}\label{eq:w_hat_grad_K_w_hat_final}
\langle\what,\,\Dij\Kyl\,\what\rangle &= \langle\what,\, (\bij\otimes\bm{e}_i + \bm{e}_i\otimes\bij)\what\rangle\nonumber\\
&= 2 \whati \langle \what,\,\bij\rangle \nonumber \\
&= 2 \sigma^{-2} \whati \Big( \langle\what,\,\Y_{j}\odot \Ki - \yij \Ki\rangle\Big) \nonumber \\
&= 2 \sigma^{-2} \whati \Big( \langle\what,\, \Y_{j}\odot \Ki\rangle - \yij\langle\what,\, \Ki\rangle \Big)
\end{align}
 For $m \in [M]$, we recall that $\nabla\kappa$ is bounded on $\X\times\X$ to prove
 \begin{align*}
\Dij \vzerov(\Y) &= \left(\Dij \vzero(\yi)\right)\bm{e}_i\\
\Dij \vzero(\yi)&= \int_{\X} \Dij \kappa( x,\yi) \dpix \\
& = \int_{\X}  (\xj-\yij)\bar{\kappa}(x,\yi) \dpix\\
& = \sigma^{-2}\Big(\int_{\X}    \xj\kappa( x,\yi) \dpix -\yij  \vzero(\yi) \Big),
 \end{align*}
where $\xj\in\RR$ is the $j$th entry of integration variable $x\in\RR^d$, and we recall that $v_{0}(y)= \int_{\X} \kappa(x,y) \dpix$.
This equivalence is precisely due to the fact that $\kappa\propto\bar{\kappa}$. For derivative kernel $\bar{\kappa}$ as discussed in \cite{BeShMa25}, this identity would not hold. We then have
\begin{equation*}
    \Dij \vzerov(\Y) =  \sigma^{-2}\left( [\vonev(\yi)]_{j}  - \yij\vzero(\yi)\right)\bm{e}_i.
\end{equation*}

Thus,
\begin{equation}\label{eq:w_hat_grad_v_final}
    \langle\what,\,\nabla_{ij} \bm{v}_0(\Y)\rangle = \sigma^{-2}\whati\left( [\vonev(\yi)]_{j}  - \yij\vzero(\yi)\right).
\end{equation}
Combining~\eqref{eq:grad_F_pi_first}, \eqref{eq:w_hat_grad_K_w_hat_final} and \eqref{eq:w_hat_grad_v_final}, we obtain
\begin{align}\label{eq:2times_gradient_FM}
    2\Dij F_{M, \lambda}(\Y) &= 2\sigma^{-2}\whati\left(-[\vonev(\yi)]_{j}  + \yij\vzero(\yi)+ \langle\what,\, \Y_{j}\odot\Ki\rangle - \yij\langle\what,\,\Ki\rangle
    \right) \nonumber \\
    &= 2\sigma^{-2}\whati\left( -[\vonev(\yi)]_{j} + \yij \vzero(\yi) - \yij \langle\what,\,\Ki\rangle + \langle\what,\,\Y_{j}\odot\Ki\rangle
    \right)
\end{align}
Now,
\begin{align}%
    \yij \vzero(\yi) - \yij \langle\what,\,\Ki\rangle &= \yij \left( -\sum\limits_{m=1}^{M}\whatms\kappa(\yi,\yms) + \int_{\X} \kappa(x,\yi) \dpix \right)\nonumber\\
   \label{eq:equiv_w_lam}& =  \lambda  \whati \yij    \\
   & = \big[ \lambda   \bm{W}(\Y)\Y \big]_{i,j}    .
\end{align}
The identity used in \eqref{eq:equiv_w_lam} is precisely a rearrangement of the fact that $\bm{K}(\Y)\hat{\bm{w}} + \lambda\hat{\bm{w}} = \vzerov(\Y)$.
Using the symmetry of $\kappa$,%
\begin{align}\label{eq:what_yj_hadamard_K}
    \langle\what,\, \Y_{j}\odot \Ki\rangle &= \sum_{m=1}^M \whatms\ymsj\kappa(\yms,\yi) \nonumber\\
    &= \sum_{m=1}^M \kappa(\yi,\yms)\whatms\ymsj \nonumber\\
    &= \langle\bm{e}_i,\, \Km(\Y)\bm{W}(\Y)\Y_j\rangle \nonumber\\
    &= \left[\,\Km(\Y)\bm{W}(\Y)\Y\right]_{i,j}
\end{align}
where $\bm{W}(\Y)\in\mathbb{R}^{M\times M}$ satisfies $[\bm{W}(\Y)]_{i,m} = \whati\delta_{i,m}$.  Combining %
\eqref{eq:what_yj_hadamard_K} and the fact that $[\vonev(\yi)]_j$ is the $(i,j)$-th entry of the matrix $\vonem(\Y)$, we conclude that
\begin{equation*}
    \left(-[\vonev(\yi)]_{j} + \yij\vzero(\yi) - \yij \langle\what,\, \Ki\rangle + \langle\what,\, \Y_{j}\odot \Ki\rangle
    \right)
\end{equation*}
is the $(i,j)$-entry of the matrix%
$-\vonem(\Y) +\bm{K}(\Y) \bmW \Y + \lambda \bmW \Y = -\vonem(\Y) +\bm{K}_{\lambda}(\Y) \bmW \Y$. Thus, \eqref{eq:2times_gradient_FM} is
the $(i,j)$-entry of the matrix
\begin{equation}
    2 \sigma^{-2} \bmW \left( \bm{K}_{\lambda}(\Y) \bmW \Y  - \vonem(\Y)\right).
\end{equation}
In other words,
\begin{align}
    \nabla F_{M}(\Y) &= \sigma^{-2}\bmW \left( \bm{K}_{\lambda}(\Y) \bmW \Y  - \vonem(\Y)\right).
\end{align}%
When the weights $\bmW$ are nonzero, we thus have
\begin{equation}
    \nabla F_{M}(\Y) = \sigma^{-2}\bmW \bm{K}_{\lambda}(\Y) \bmW\left( \Y - \bmW^{-1} \bm{K}_{\lambda}(\Y)^{-1} \vonem(\Y) \right).
\end{equation}
We remark that $\Kyl$ is invertible by construction for any $\lambda > 0$.
Finally, observe that the function $\bm{\Psi}_{\lambda,\MMS}$ defined by~\eqref{eqn:def_msip_lambda} satisfies
\begin{equation}
\bm{\Psi}_{\MMS}(\Y) = \bmW^{-1} \bm{K}_{\lambda}(\Y)^{-1} \vonem(\Y).
\end{equation}

\subsection[Proof of MSIP invariance]{Proof of~\Cref{prop:msip_invariance}}
First, fix $C > 0$ and define scaled density $\tilde{\pi} = C\pi$ for convenience. Then, given $\Y\in\gX^M$, we have %
\begin{equation}
\Psi_{\MSIP, \lambda}(Y;\pi) = \WLY^{-1}\bm{K}_{\lambda}(\Y)^{-1}\bm{V}_{0}(Y)r(Y),
\end{equation}
where $\bm{V}_{0}(Y) = \mathrm{diag}(\vzerob(Y))$, and $r:\gX^M\to\gX^M$ has row $i$ defined by $v_{1}(y_i)/v_{0}(y_i)$.
Fix $\Y\in\gX^M$, then define the functions $\vzero^C$ and $\vone^C$ for $\tilde{\pi}$, corresponding to $\vzero$ and $\vone$, respectively. Remarkably, we observe
\begin{gather}
    \vzero^C(y) = C\int\kappa(x,y)\dpix = C\vzero(y),\\
    \vone^C(y) = C\int x\kappa(x,y)\dpix = C\vone(y).
\end{gather}
This identity also holds for any of our estimators $\hat{v}_0$ and $\hat{v}_1$. Thus,
\begin{equation}
    \frac{v_{1}^{C}(y)}{v_{0}^{C}(y)} = \frac{v_{1}(y)}{v_{0}(y)}.
\end{equation}
Stated plainly, the vector of ratios $r(Y)$ is invariant to the value of $C$. Moreover, the weights $\hat{\bm{w}}_{\lambda}^C$ for $\tilde{\pi}$ satisfy
    \begin{equation}
\bm{w}_\lambda^{^\ast C}(\Y) := \bm{K}_{\lambda}(Y)^{-1}\vzerob^C(\Y) = C\wlY,
    \end{equation}
so that the product $\WLY^{-1}\bm{K}_{\lambda}(\Y)^{-1}\bm{V}_{0}(Y)$ is invariant to the value of $C$. Therefore, $\Psi_{\mathrm{MSIP}, \lambda}(Y; C\cdot \pi)$ is invariant to the value of $C$.

\subsection{One-point MSIP map as a GMM approximation}\label{proof:gmm_msip_stein}
\begin{proposition}\label{prop:gmm_msip_stein}
    Suppose the target density $\pi$ is in RKHS $\mathcal{H}$ with squared-exponential kernel. Evaluate $\widehat{\Psi}_{\mathrm{MSIP},\lambda}$ using one-point estimators $\hvzero^1$ and $\hvonev^{S1}$ in~\eqref{eqn:one_point_defs}. %
    If there exists a measure $\mu$ with
    \begin{equation}
        \pi(x) = \int\kappa(x,y)\ \mathrm{d}\mu(y),
    \end{equation}
    then any fixed point of $\widehat{\Psi}_{\mathrm{MSIP},\lambda}$ is a critical point of $F_{M,\lambda}^{\mu}$. Further, we have the identity
    \[
    F^\mu_{M,\lambda}(\Y) \equiv \min_{\bm{w}\in\mathbb{R}^M} \left\|\pi - \sum_{i=1}^M w_i \mathcal{N}(\cdot;y_i,\ \sigma^2 I)\right\|^2_{\mathcal{H}} + \frac{\lambda}{2}\|\bm{w}\|^2.
    \]
\end{proposition}
\begin{proof}
We remark that the MSIP map $\Psi_{\MSIP,\lambda}$ defined using $\vzerob$ and $\vonem$ for the measure $\mu$ precisely aligns with the practical calculation of $\hvzero^1$ and $\hvonev^{S1}$. Therefore, we apply \cref{thm:lambda_MSIP_preconditioner} to arrive at the first half of the result.

Now, consider a configuration of points and weights $Y,\bm{w}$. Let $f_i(x) = \mathcal{N}(x;y_i,\sigma^2 I)$ and note that
\begin{align*}
    \pi(x) - \sum_{i=1}^M w_i f_i(x) = \left\langle \kappa(x,\cdot), \mu - \sum_{i=1}^M w_i\delta_{y_i}\right\rangle = \Phi_\mathcal{H}\left[\mu - \sum_{i=1}^M w_i\delta_{y_i}\right](x),
\end{align*}
where $\Phi_{\mathcal{H}}$ is the feature map of RKHS $\mathcal{H}$~\cite{MuFuSrSc17}%
Therefore, we see that
\[
\left\|\pi - \sum_{i=1}^M w_i \mathcal{N}(\cdot;y_i,\ \sigma^2 I)\right\|^2_{\mathcal{H}} = \left\| \Phi_\mathcal{H}\left[ \mu - \sum_{i=1}^M w_i\delta_{y_i}\right]\right\|^2_{\mathcal{H}} = \MMD^2\left(\mu, \sum_{i=1}^M w_i \delta_{y_i}\right).
\]
The result thus follows from the definition of $F^\mu_{M,\lambda}$.
\end{proof}

\subsection{2D visualisations}\label{sec:two_dim_viz}
In the following figures, we include a slew of visualizations of final configurations of our proposed methods compared to others. The algorithms include the ones benchmarked against in the main paper, as well as multi-start gradient ascent (GA). One can see that GA has significant problems capturing anisotropy within the modes. Further, we distinguish between SVGD with a fixed bandwidth (here, denoted simply SVGD) versus SVGD with an adaptive bandwidth (a-SVGD) through the median heuristic, which is used in the main experiments.

\begin{figure}
    \centering
    \includegraphics[width=0.95\linewidth]{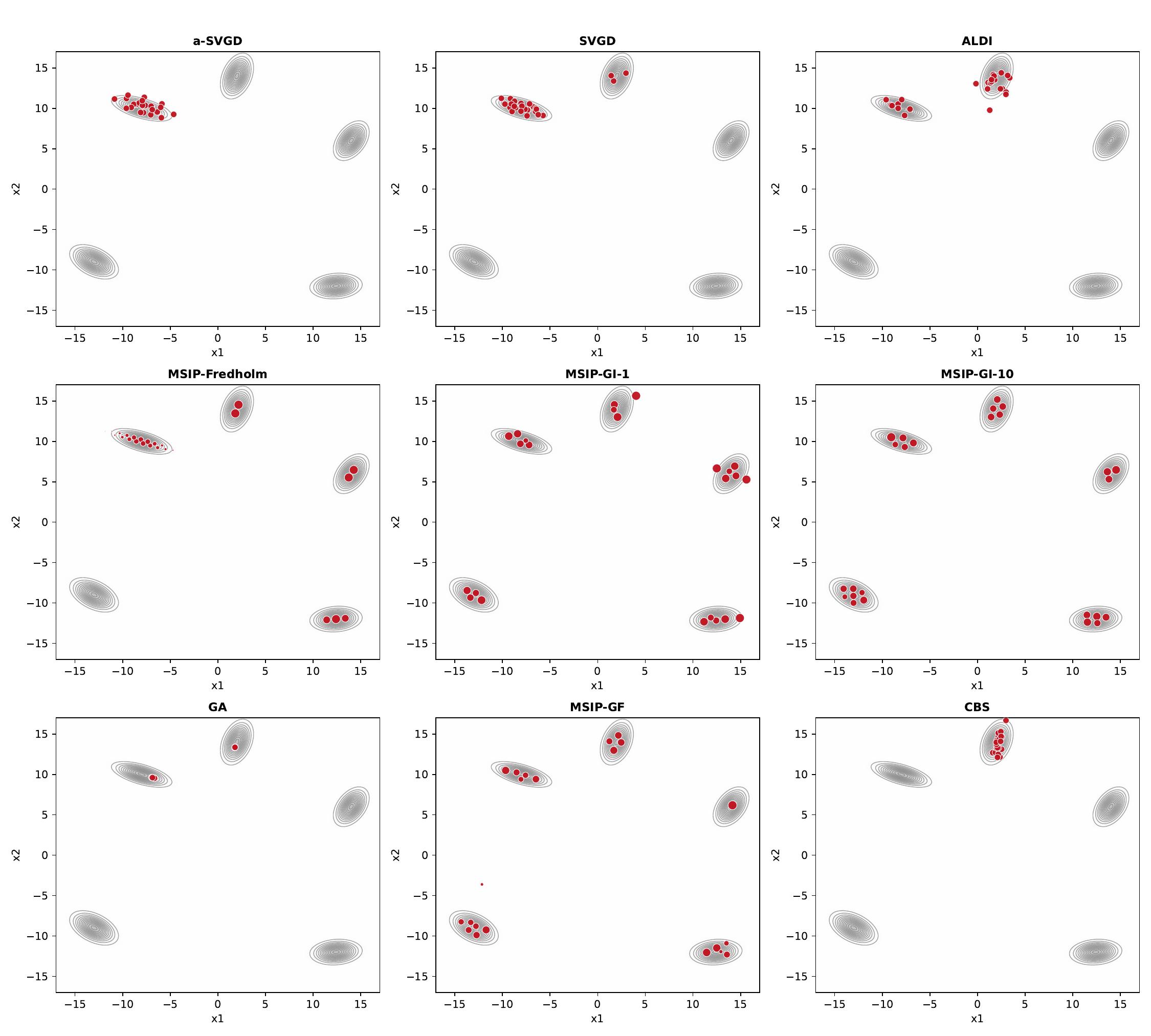}
    \caption{GMM 2D}
    \label{fig:2dvis_gmm2d}
\end{figure}
\newpage

\begin{figure}
    \centering
    \includegraphics[width=0.95\linewidth]{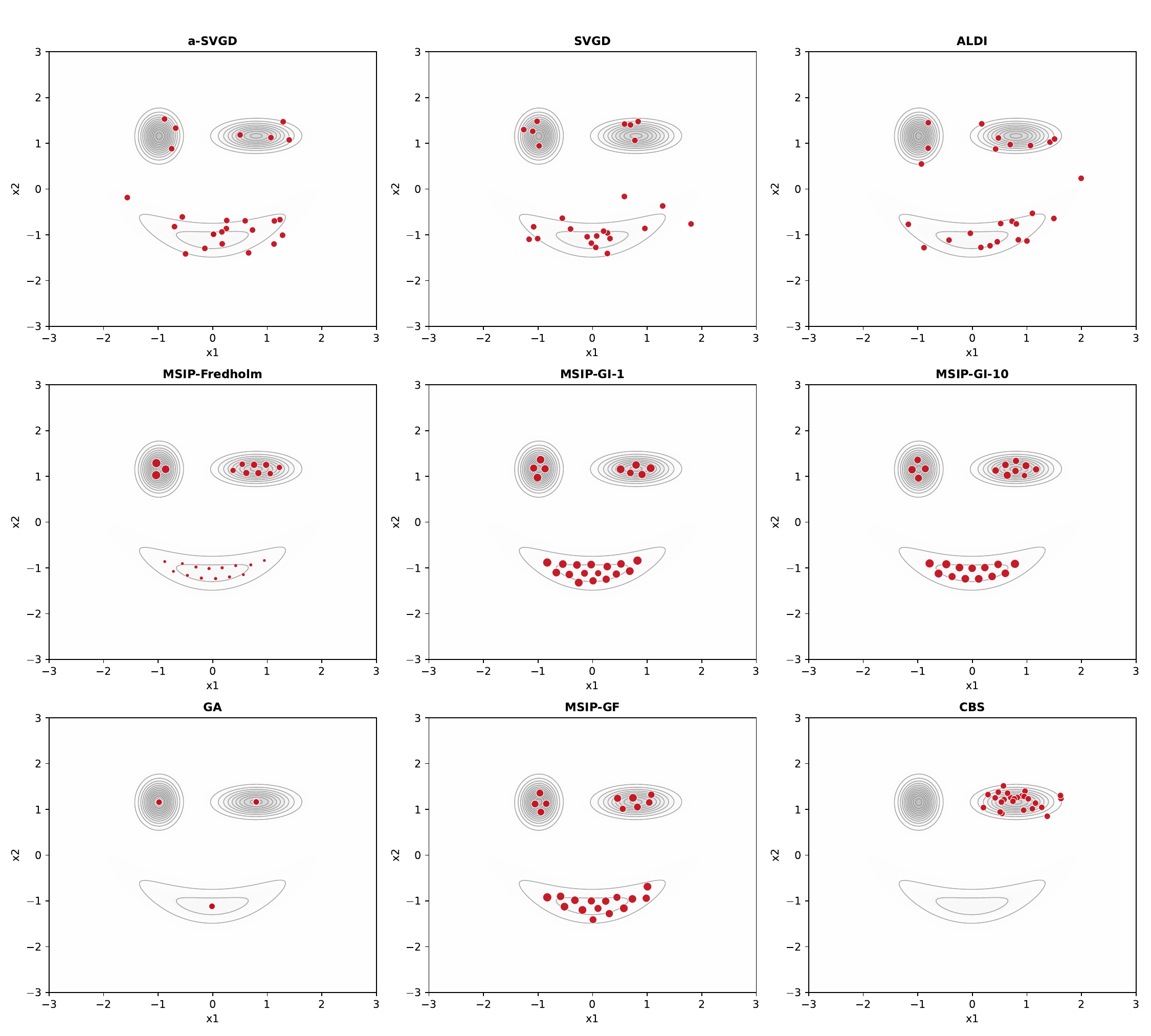}
    \caption{Joker}
    \label{fig:2dvis_joker}
\end{figure}
\newpage
\begin{figure}
    \centering
    \includegraphics[width=0.95\linewidth]{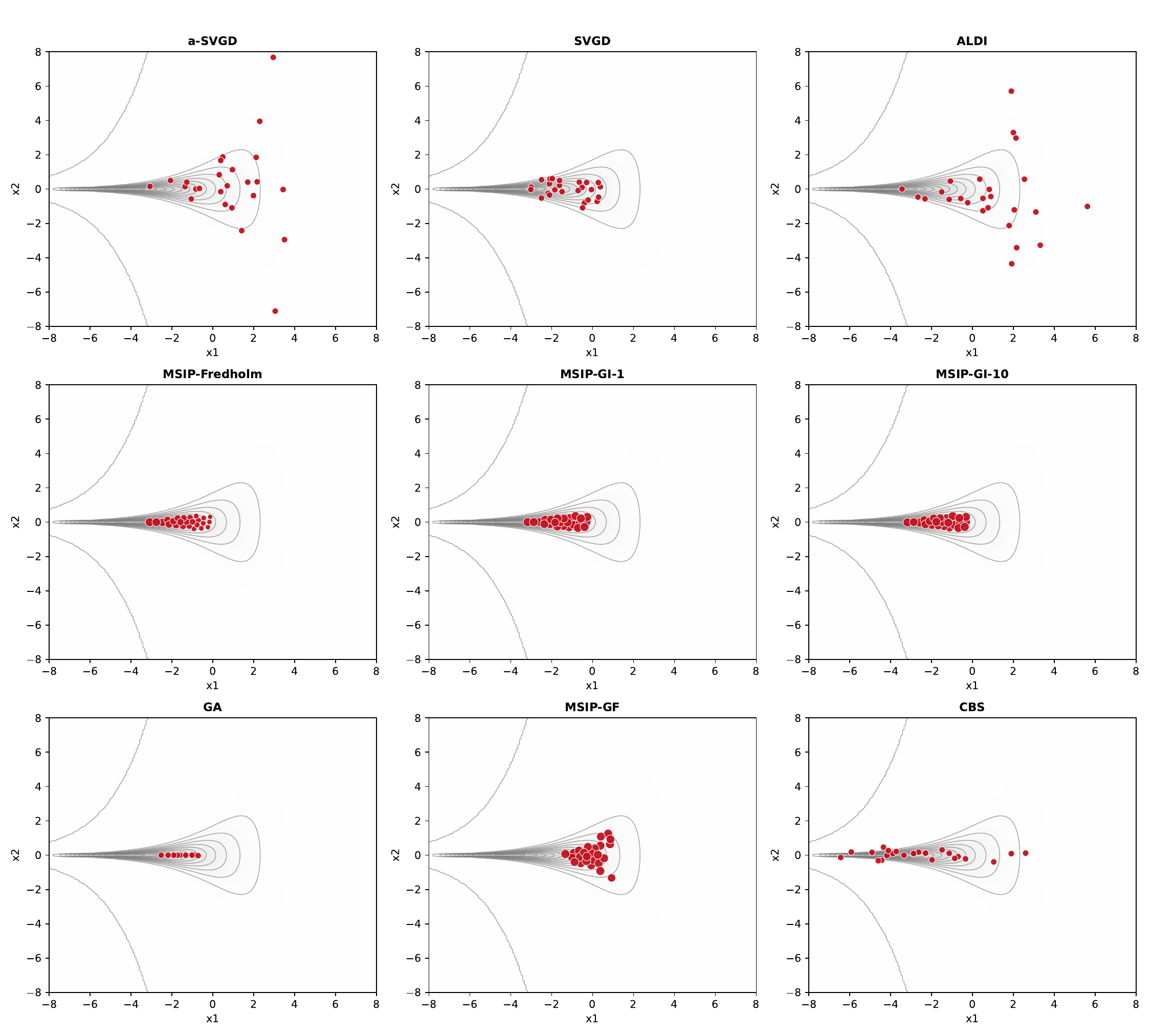}
    \caption{Funnel2D}
    \label{fig:2dvis_funnel2d}
\end{figure}
\newpage
\begin{figure}
    \centering
    \includegraphics[width=0.95\linewidth]{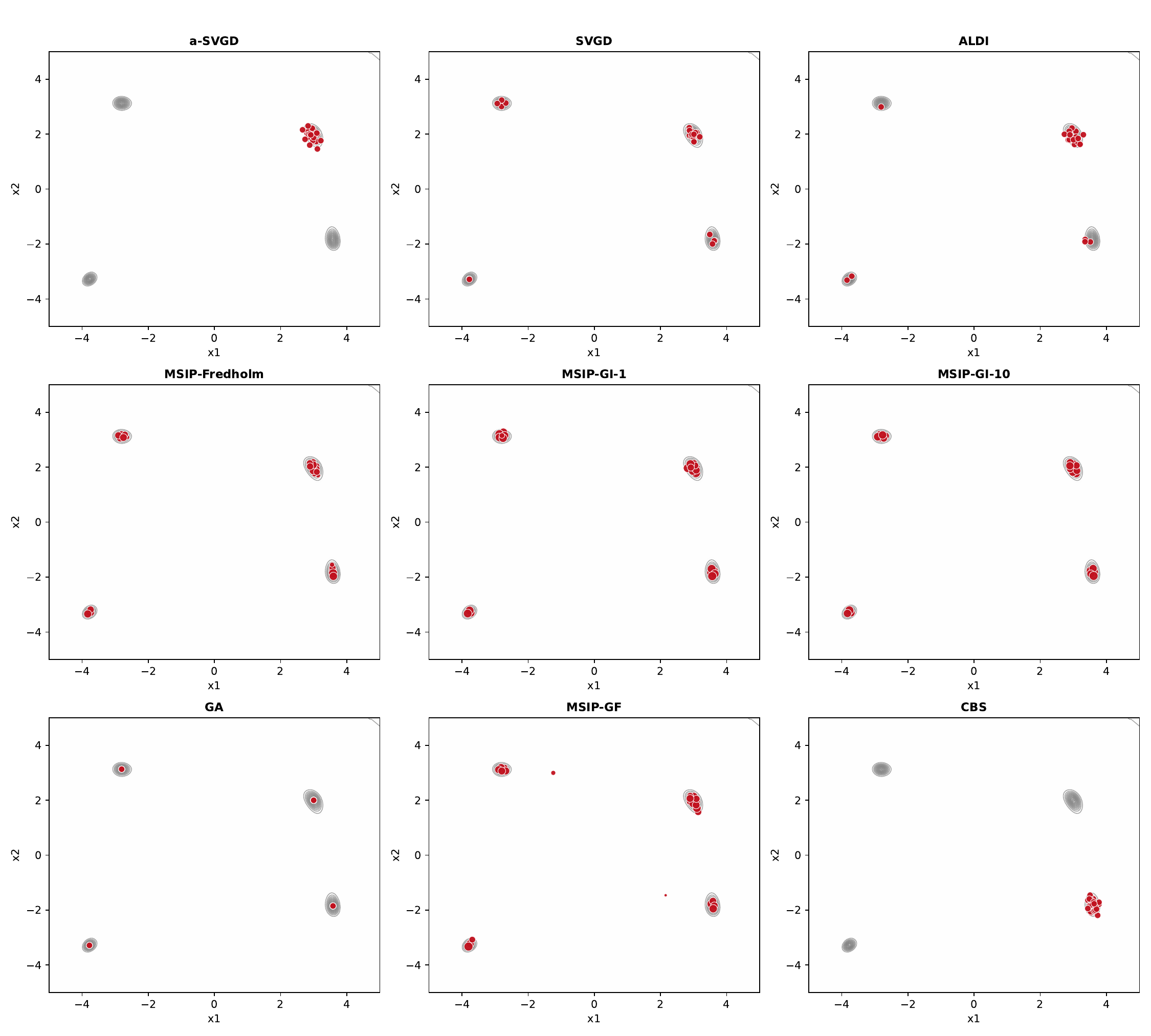}
    \caption{Himmelblau}
    \label{fig:2dvis_himmelblau}
\end{figure}

\subsection{Additional numerics}

\begin{table}[h]
{

\caption{Error of different algorithms for synthetic distributions using KSD with an inverse multi-quadric kernel for \emph{unweighted particles}. We report mean and standard deviation over many trials; the smallest row-wise mean is bolded. Empty cells correspond to an algorithm not converging.}
    \tiny
    \centering
    \label{tab:ksd_imq_sigma_0.5_unweighted}
    \begin{tabular}{@{}l c c c c c c c c@{}}\toprule
    & & \multicolumn{5}{c}{Gradient-informed algorithms} & \multicolumn{2}{c}{Gradient-free algorithms} \\
    \cmidrule(lr){3-7}\cmidrule(lr){8-9}
     & Dim. & \textbf{MSIPF} & \textbf{MSIPGI-1} & \textbf{MSIPGI-10} & SVGD & ALDI & \textbf{MSIPGF} & CBS \\\midrule
    GMM2 & 2 & \textbf{0.794 (0.025)} & 1.249 (0.366) & 1.658 (2.049) & 1.223 (0.039) & 0.930 (0.094) & 25.54 (22.15) & 1.103 (0.090) \\
    GMM5 & 5 & 1.378 (0.029) & 3.014 (3.674) & \textbf{1.310 (0.001)} & 1.499 (0.004) & 1.516 (0.054) & 19.95 (40.99) & 3.293 (0.336) \\
    GMM10 & 10 & 2.228 (9.064e-04) & \textbf{1.965 (0.010)} & 1.972 (0.022) & 2.172 (0.002) & 2.185 (0.039) & 8.774 (8.066) & 3.899 (0.369) \\
    GMM20 & 20 & 3.395 (8.309e-06) & 3.018 (0.031) & \textbf{2.912 (0.044)} & 3.060 (0.003) & 3.355 (0.171) & 2.938 (0.018) & 5.061 (0.345) \\
    Joker & 2 & 8.447 (0.028) & 8.401 (0.016) & \textbf{8.394 (0.026)} & 11.11 (1.579) & 9.286 (0.185) & 11.98 (7.794) & 10.02 (1.475) \\
    Funnel10 & 10 & 20.24 (0.041) & 20.24 (0.039) & 20.24 (0.040) & 24.06 (1.593) & 20.34 (0.167) & \textbf{20.20 (0.026)} & 21.86 (1.396) \\
    Funnel20 & 20 & 28.93 (0.140) & 28.92 (0.163) & 28.94 (0.165) & 67.36 (22.99) & \textbf{28.66 (0.122)} & 28.74 (0.101) & 34.05 (3.475) \\
    Funnel 2D & 2 & 8.628 (0.021) & 8.606 (0.023) & \textbf{8.600 (0.020)} & 9.073 (0.145) & 9.008 (0.091) & 8.755 (0.019) & 95.40 (176.8) \\
    Funnel50 & 50 & 47.78 (2.888) & 46.70 (0.741) & 46.82 (0.785) & 242.5 (142.7) & --- (---) & \textbf{45.54 (0.132)} & 58.30 (10.52) \\
    Funnel5 & 5 & 14.25 (0.021) & 14.25 (0.021) & 14.25 (0.020) & 18.64 (2.574) & 14.43 (0.210) & \textbf{14.24 (0.017)} & 14.98 (0.615) \\
    Himmelblau & 2 & 7.792 (0.239) & 9.050 (1.389) & \textbf{7.389 (0.080)} & 14.24 (0.311) & 14.33 (2.558) & 8.971 (2.273) & --- (---) \\
    \bottomrule
    \end{tabular}
}
\end{table}